\documentclass[10pt,twocolumn,letterpaper]{article}

\usepackage{iccv}
\usepackage{times}
\usepackage{epsfig}
\usepackage{graphicx}
\usepackage{amsmath}
\usepackage{amssymb}

\usepackage{wrapfig}
\usepackage[utf8]{inputenc} 
\usepackage[T1]{fontenc}    
\usepackage{url}            
\usepackage{booktabs}       
\usepackage{amsfonts}       
\usepackage{nicefrac}       
\usepackage{microtype}      

\usepackage{amssymb}
\usepackage{graphicx}
\usepackage{multirow}
\usepackage{amsthm}
\usepackage{amsmath}
\usepackage[usenames,dvipsnames]{xcolor}
\usepackage[lined,boxed,commentsnumbered]{algorithm2e}
\usepackage{enumerate}
\usepackage{verbatim}
\usepackage{tabularx}

\usepackage{amsmath,amsfonts,bm,amsthm}

\newcommand{\set}[1]{\left\{#1\right\}}

\newcommand{\brac}[1]{\left [#1\right ]}

\newcommand{\Real}{\mathbb R}

\newcommand{\too}{\rightarrow}

\makeatletter
\newtheorem*{rep@theorem}{\rep@title}
\newcommand{\newreptheorem}[2]{%
\newenvironment{rep#1}[1]{%
 \def\rep@title{#2 \ref{##1}}%
 \begin{rep@theorem}}%
 {\end{rep@theorem}}}
\makeatother

\newtheorem{theorem}{Theorem}
\newreptheorem{theorem}{Theorem}

\newreptheorem{lemma}{Lemma}

\newtheorem{definition}{Definition}

\makeatletter
\newcommand{\subalign}[1]{%
  \vcenter{%
    \Let@ \restore@math@cr \default@tag
    \baselineskip\fontdimen10 \scriptfont\tw@
    \advance\baselineskip\fontdimen12 \scriptfont\tw@
    \lineskip\thr@@\fontdimen8 \scriptfont\thr@@
    \lineskiplimit\lineskip
    \ialign{\hfil$\m@th\scriptstyle##$&$\m@th\scriptstyle{}##$\crcr
      #1\crcr
    }%
  }
}
\makeatletter

\newcommand*\samethanks[1][\value{footnote}]{\footnotemark[#1]}

\def\Eqref#1{Equation~\ref{#1}}

\def\Algref#1{Algorithm~\ref{#1}}

\def\1{\bm{1}}

\DeclareMathAlphabet{\mathsfit}{\encodingdefault}{\sfdefault}{m}{sl}
\SetMathAlphabet{\mathsfit}{bold}{\encodingdefault}{\sfdefault}{bx}{n}

\newtheorem{claim}{Claim}
\usepackage[breaklinks=true,bookmarks=false]{hyperref}

\iccvfinalcopy 

\def\iccvPaperID{4047} 

\begin{document}
	
	\title{Surface Networks via General Covers}
	
	\author{Niv Haim \thanks{Equal contribution}\\
		\and
		Nimrod Segol \samethanks \\
		\and 
		Heli Ben-Hamu \\
		\and 
		Haggai Maron \\
		\and
		Yaron Lipman \\
		\\
		Weizmann Institute of Science\\
		Rehovot, Israel\\
	}

	\maketitle

	\begin{abstract}
		Developing deep learning techniques for geometric data is an active and fruitful research area. This paper tackles the problem of sphere-type surface learning by developing a novel surface-to-image representation. Using this representation we are able to quickly adapt successful CNN models to the surface setting. 
		
		The surface-image representation is based on a covering map from the image domain to the surface. Namely, the map wraps around the surface several times, making sure that every part of the surface is well represented in the image. Differently from previous surface-to-image representations, we provide a low distortion coverage of all surface parts in a single image. Specifically, for the use case of learning spherical signals, our representation provides a low distortion alternative to several popular spherical parameterizations used in deep learning.
		
		We have used the surface-to-image representation to apply standard CNN architectures to 3D models including  spherical signals. We show that our method achieves state of the art or comparable results on the tasks of shape retrieval, shape classification and semantic shape segmentation.
		
	\end{abstract}
	

	\vspace{-17pt}
	\section{Introduction}
	Adapting deep learning methods to geometric data (\eg, shapes) is a vibrant research area that has already produced state of the art algorithms for several geometric learning tasks (\eg,  \cite{qi2017pointnet,qi2017pointnet++,su2015multi}). 
	
	Two prominent approaches are: (i) mapping the geometric data to tensors (\eg, images) and using off-the-shelf convolutional neural network (CNN) architectures and optimization techniques \cite{su2015multi,Wu_2015_CVPR,sinha2016deep,maron2017convolutional}; and (ii) developing novel architectures and optimization techniques that are tailored to the geometric data \cite{masci2015geodesic,qi2017pointnet,qi2017pointnet++}. 
	An important benefit of (i) is in reducing the geometric learning task to an image learning one, allowing to harness the \emph{huge} algorithmic progress of neural networks for images directly to geometric data. 
	
	Some previous attempts, following (i), to perform learning tasks on geometric data use projections to 2D planes, \eg, by rendering the shapes \cite{su2015multi}. Such projections are not injective and suffer from occlusions, thus often require a collection of projections for a single shape. Other methods embed the shape in an encapsulating 3D grid \cite{Wu_2015_CVPR,maturana2015voxnet}; these methods require dealing with higher dimensional tensors and are usually less robust to deformations. Other methods \cite{sinha2016deep,maron2017convolutional} try to find low distortion 2D mappings to an image domain. In this case the intrinsic dimensionality of the data is preserved, however, these maps suffer from high distortion and/or ignore the difference in the topologies of the surface (no boundary) and the image (with boundary).
	
	In this paper, we advocate a novel 2D mapping method for representing sphere-type (genus zero, \eg, the human model in Figure \ref{fig:E0}a, left) surfaces as images. The challenge in using an image to represent a surface has two aspects: geometrical and topological. Geometrically, a general curved surface cannot be mapped to a flat domain (\ie, the image) without introducing a significant distortion. Topologically, an image has a boundary while sphere-type surfaces do not; hence, any mapping between the two will introduce cuts and discontinuities. Furthermore, a naive application of 2D convolution to the image would be ambiguous on the surface (see Figure \ref{fig:conv_on_flat_spheres} and Subsection \ref{ss:conv_on_flat_spheres}). 
	
	To address these challenges we think of the image as a periodic domain (\ie, a torus) and relax the notion of a one-to-one mapping to that of a \emph{covering map} from the image domain onto the surface. That is, we construct a mapping from the image domain to the surface that covers the surface several times. For example, Figure \ref{fig:E0}a visualizes a degree-$5$ covering map. Meaning, the surface appears $5$ times in the image; note how each part of the surface appears with low distortion at-least once in the image. The image generated by our covering map is periodic, namely its left and right boundaries as well as its bottom and top boundaries correspond, making the image boundaryless. Importantly, since image convolution is well defined on a torus, it will translate to a continuous convolution-like operator on the surface \cite{maron2017convolutional}.

    Applying our method to surface learning is easy: use a covering map to transfer functions of interest over the input surfaces (\eg, the coordinate functions) to images and apply one's favorite CNN with periodic padding. 
		
	We tested our method in two scenarios: spherical signal learning \cite{coors2018spherenet,cohen2018spherical}, and surface collection learning. For spherical signal learning, our approach provided state of the art results among all spherical methods on a shape retrieval dataset (SHREC17 \cite{savvashrec}) and a shape classification dataset (ModelNet40 \cite{Wu_2015_CVPR}). For surface collection learning, our method produced state of the art results on a surface segmentation dataset (Humans \cite{maron2017convolutional}). 
	Our contributions are:
	\begin{itemize}\vspace{-6pt}
		\item We introduce a broad family of low distortion surface-to-toric image representations. The toric image representation allows applying off-the-shelf CNNs to general genus-zero surfaces. \vspace{-3pt}
		
		\item In particular, we provide a framework for learning spherical signals using CNNs. \vspace{-6pt}
		
		\item We introduce a practical algorithm for computing toric covers of genus zero surfaces. \vspace{-4pt}
	\end{itemize}

	Our code is available at \url{https://github.com/nivha/surface_networks_covers}

	\section{Previous work}	\vspace{-3pt}

		\begin{figure}[t]
		\begin{tabular}{@{\hskip3pt}c@{\hskip3pt}}
			\includegraphics[width=\linewidth]{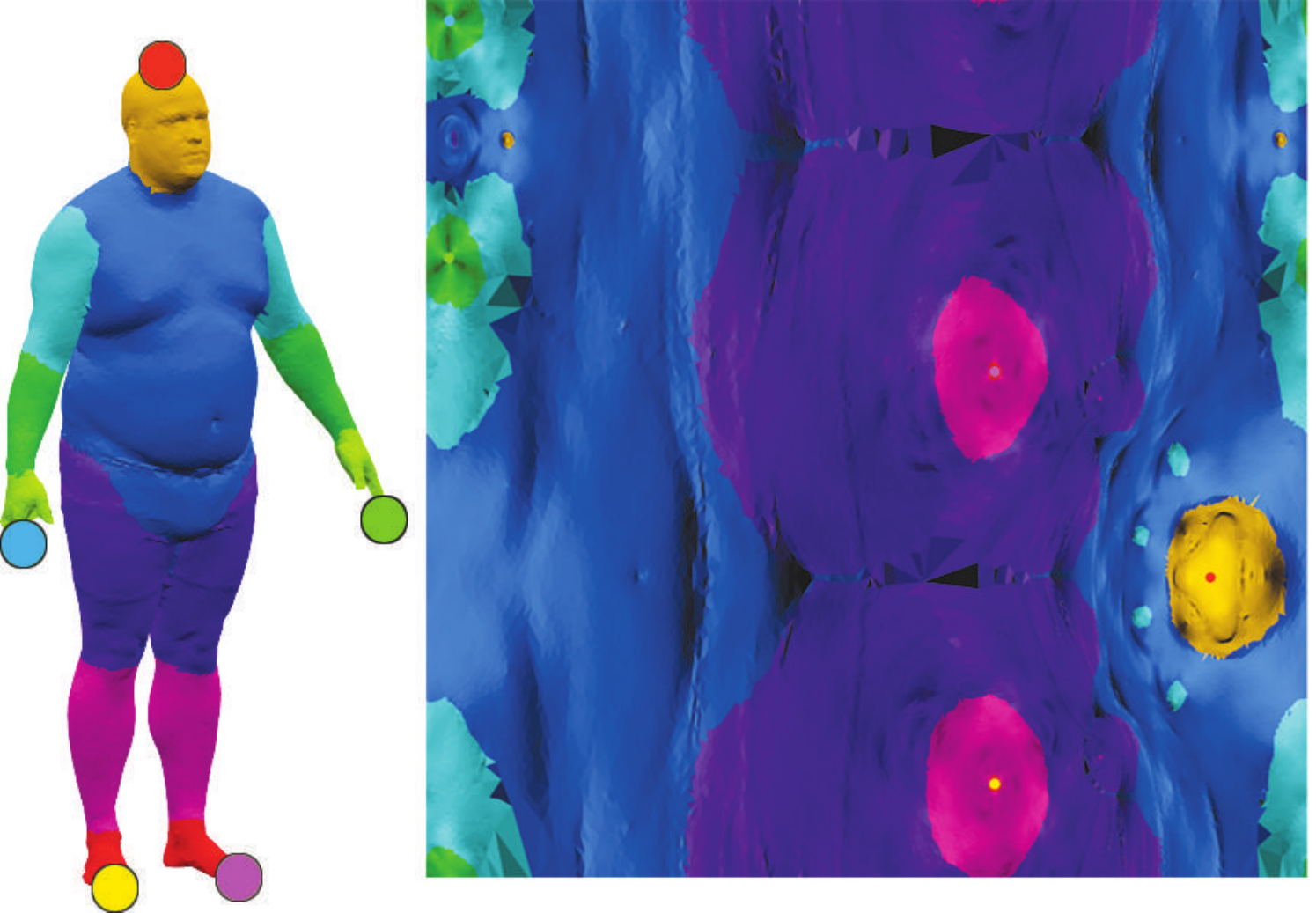} \vspace{-5pt} \\ (a)\vspace{5pt} \\ 
			\includegraphics[width=\linewidth]{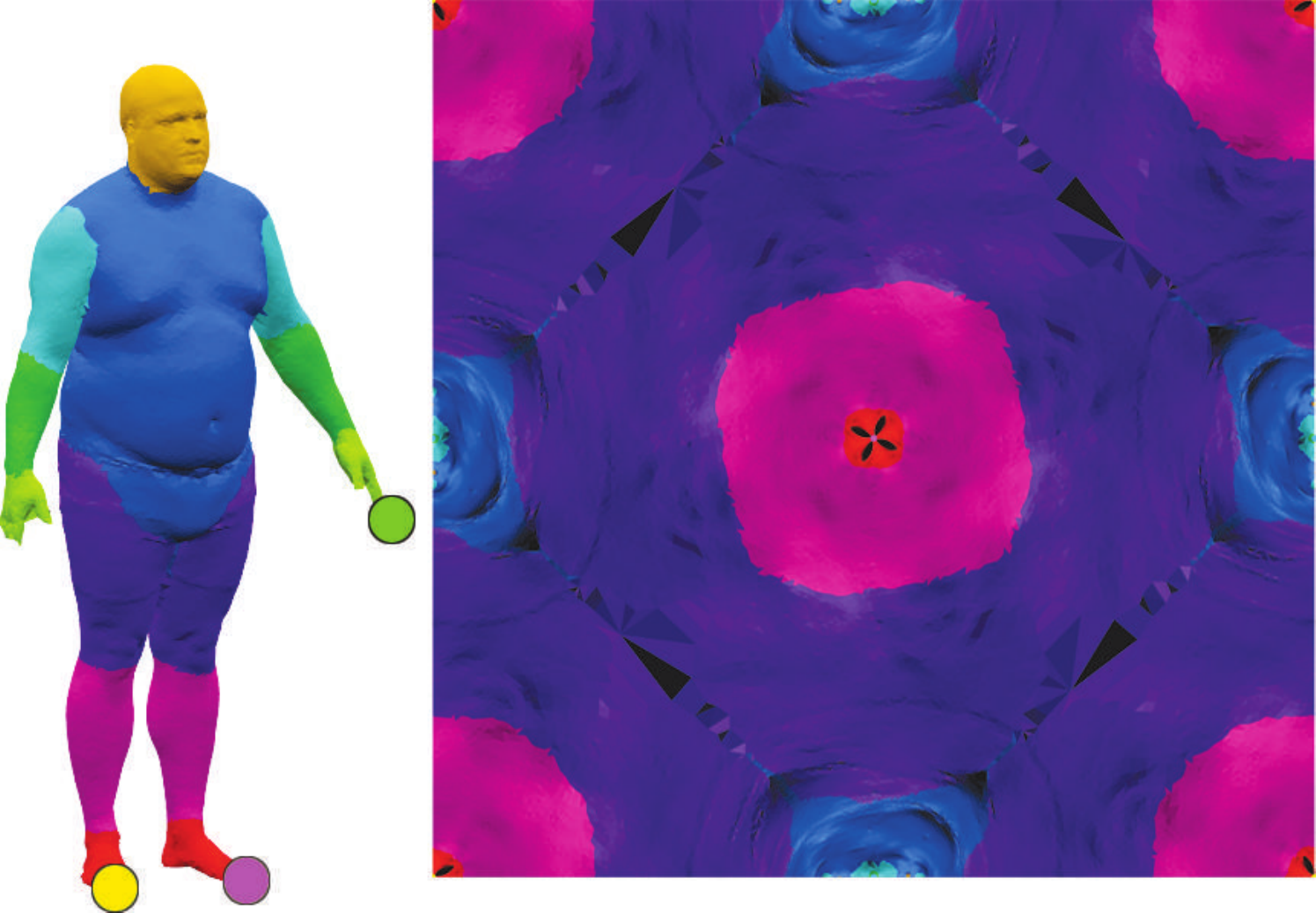} \vspace{-5pt}\\
			(b)  \vspace{1pt}\\
		\end{tabular}
		\caption{(a) A periodic general cover with $k=5$ branch points (colored dots) and degree $d=5$ of a human model, computed using our method. (b) The orbifold cover of \cite{aigerman2015orbifold} with $k=3$ branch points and degree $d=4$. Note that the cover in (b) has four (rotated) repetitions of the same mapping; is missing the head and right arm; and the torso (blue) suffers from a considerable down-scale. \vspace{-10pt}}\label{fig:E0}
	\end{figure}
	
	Applying deep learning techniques to geometric data has proved to be a huge success in the last few years.  A wide variety of methods were suggested, where the most popular approaches are: volumetric based methods (\eg, \cite{Wu_2015_CVPR,maturana2015voxnet}), rendering based methods (\eg, \cite{su2015multi,wei2016dense,yu2018multi}), spectral based methods (\eg, \cite{bruna2013spectral,defferrard2016convolutional}) and methods that operate directly on the surface itself (\eg, \cite{masci2015geodesic}). A popular related problem is the problem of learning on point clouds which received a lot of attention lately (see \eg, \cite{qi2017pointnet,atzmon2018point,li2018pointcnn}).
	
	Here, we restrict our attention to intrinsic or parameterization-based surface methods and refer the reader to the above mentioned works and a recent survey \cite{bronstein2017geometric} for further information. 
	
	\textbf{Local parameterization.} Such methods (\eg,  \cite{masci2015geodesic,boscaini2016learning,monti2017geometric}) extract local surface patches and use them in order to learn point representations. In \cite{masci2015geodesic} the authors use local polar coordinates as the patch operator. In a follow-up work,  \cite{boscaini2016learning} use projections on oriented anisotropic diffusion kernels, where \cite{monti2017geometric} learn the patch operator using a Gaussian mixture model. In contrast to these works, we employ a global parameterization which represents the shape using a single image.
	
	\textbf{Global parameterization.} Other methods use global parametrization of the surface to a canonical domain. \cite{sinha2016deep,sinha2017surfnet} use an area-preserving parameterization and map surfaces to a planar domain (going through a sphere); the global area-preserving parameterization cannot cover the surface with low distortion everywhere and depends on the specific cut made on the surface. 
	
	The most similar method to ours is \cite{maron2017convolutional} that proposes gluing four copies of the surface into a torus and map it conformally (\ie, preserving angles) to a flat torus, where the convolution is well defined. Their map is defined by a choice of three points on the surface, and suffers from significant angle and scale distortion, see Figure \ref{fig:E0}b (\eg, the head, right arm and torso). In order to cover each point on the surface reasonably well, the authors sample multiple triplets of points from each surface where each triplet focuses on a different part of the surface. In a follow up work, \cite{hamu2018multi} use the same parameterization as a surface representation for Generative Adversarial Networks (GANs) \cite{goodfellow2014generative}. In order to deal with the high distortion of each single parameterization, the authors devise a multi-chart structure and rely on given sparse correspondences between the surfaces.

	\textbf{Convolutions on tangent planes.} \cite{tatarchenko2018tangent} define convolutions on surfaces by working on the tangent planes. \cite{pan2018convolutional} also define the convolutions on tangent planes and relate convolutions on nearby points using parallel transport. \cite{poulenard2018multi} define convolutions on surfaces by extending the notion of a signal on a surface into a directional signal and build layers that are equivariant to the choice of reference directions. \cite{huang2019texturenet} utilizes 4-rotational symmetric field to define a domain for convolution on a surface.
	
	\textbf{Convolutions of spherical signals.} Our work targets learning of general genus zero surfaces. In particular, it can facilitate learning of spherical signals, a task that has received growing interest in the last few years. \cite{su2017learning, coors2018spherenet, zhao2018distortion} note that an equirectangular projection of a spherical signal suffers from large distortions and suggest network architectures that try to compensate for these distortions. \cite{cao20173d} perform 2D convolution on spherical strips extracted from the spherical signal. 
	\cite{jiang2019spherical} suggest to define the convolution of a spherical signal as a linear combination of differential operators with learned weights.   
	In a different line of work, \cite{cohen2018spherical,esteves2018learning,kondor2018clebsch} propose networks that are invariant to the natural action of $SO(3)$ on spherical signals. \cite{cohen2019gauge} advocate the notion of gauge equivariance as the correct equivariance notion on manifolds, and construct gauge equivariant networks on spheres. 
	
	\textbf{Other methods.}	\cite{xu2017directionally} tackle the shape segmentation problem by a novel architecture that operates on local features (such as normals) and global features (such as distances) and then fuses them together. \cite{kostrikov2017surface} propose an improved graph neural network model based on the Dirac operator.

	\vspace{-8pt}
	
	\section{Preliminaries} \vspace{-6pt}
	In this section we discuss our choice of periodic images (\ie images with toric topology) and introduce branched covering maps, the main mathematical tool used in our approach. 
	
	\begin{figure}
	    \centering
	    \begin{tabular}{ccc}
	        \includegraphics[width=0.3\columnwidth]{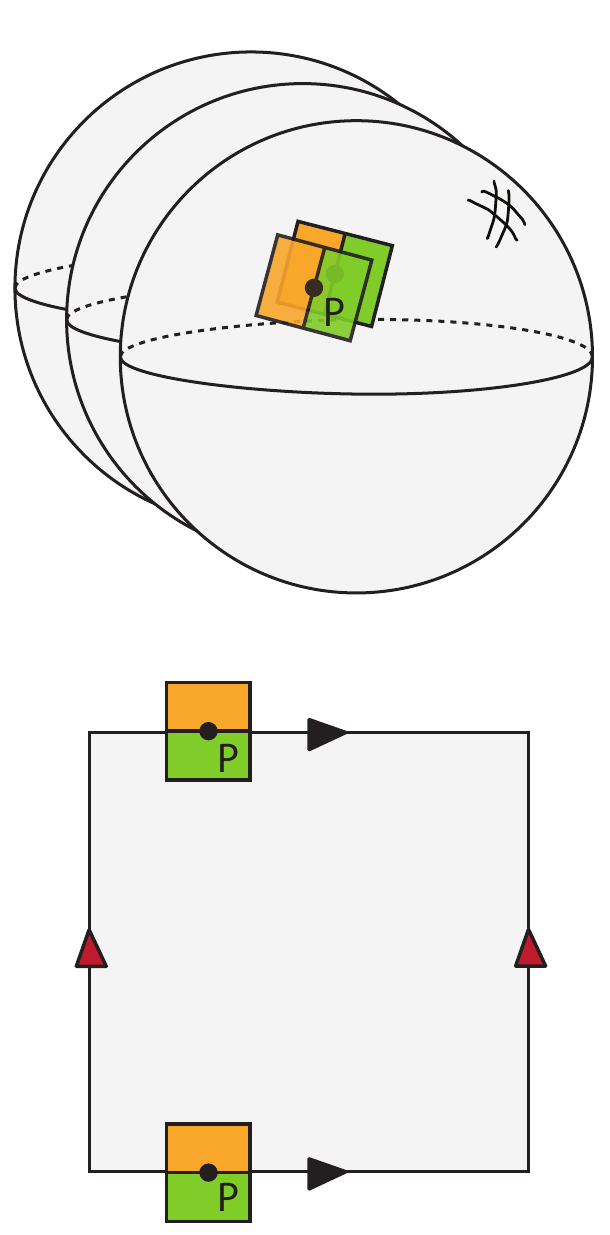} &
	        \includegraphics[width=0.3\columnwidth]{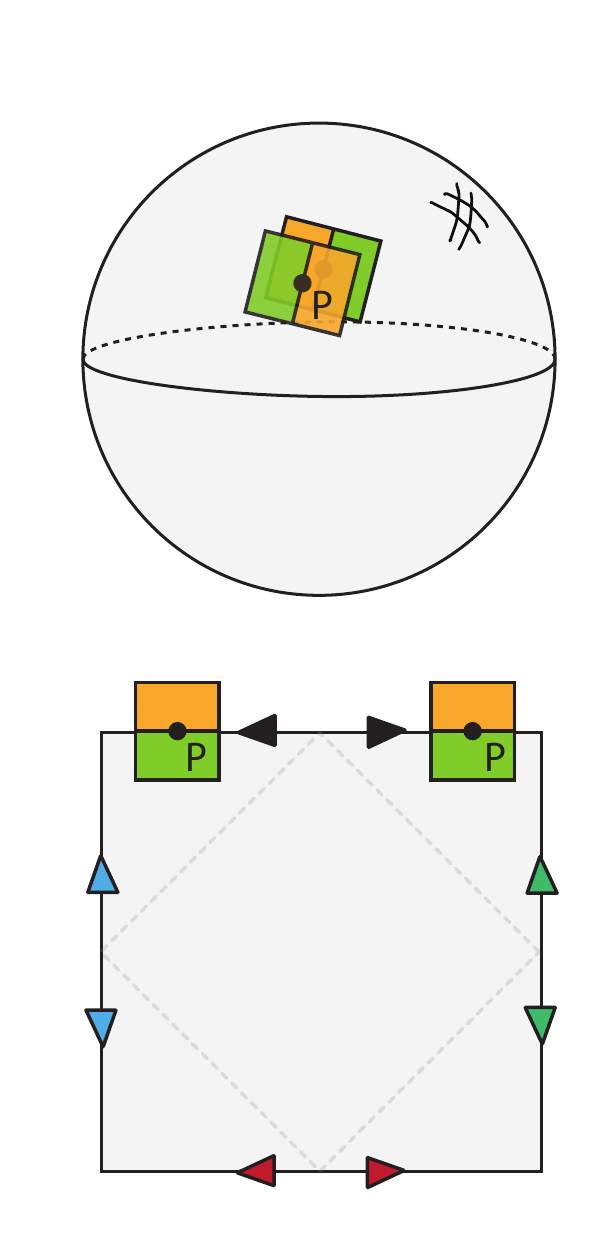}&
	        \includegraphics[width=0.3\columnwidth]{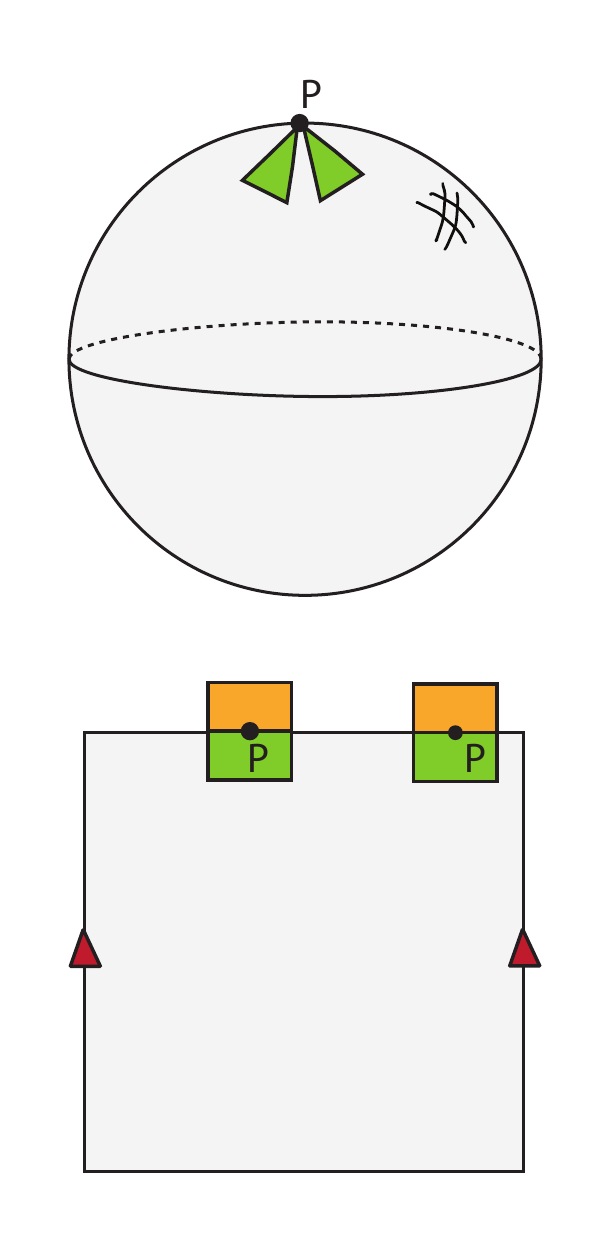} \\
	        Torus & Octahedron & Equirectangular \\ \hspace{3pt}
	    \end{tabular}
	    \caption{Standard 2D convolution applied to image representations (bottom row) of spherical topology (top row). In each example, the points indicated by $\mathrm{P}$ represent the same point on the sphere. Only for the toric topology the convolution in the image domain defines a consistent, continuous operator everywhere on the sphere. } \vspace{-10pt}
	    \label{fig:conv_on_flat_spheres}
	\end{figure}
	
	\subsection{Convolutions on flattened spheres}\vspace{-4pt}
	\label{ss:conv_on_flat_spheres}
	A standard way to apply CNNs to a signal on a sphere-type surface is to represent it as an image and apply standard 2D convolution. Since representing a sphere as an image requires cutting and duplicating the cuts, different boundary segments in the image represent the same segment on the sphere. 
	
	In the case where the transformation in the image domain between the two duplicated boundary segments is a pure translation then the result of applying 2D convolution at any two matching points on these segments will result in exactly the same value. In other cases, such as equirectangular spherical projection \cite{su2017learning} or octahedron spherical projection \cite{praun2003spherical,sinha2016deep}, 2D convolution on two matching points result in two different values. Figure \ref{fig:conv_on_flat_spheres} shows an example where duplicated image boundary segments are marked with the same color arrows; a pair of matching points (marked $\mathrm{P}$) are shown in each example along with an illustration of a convolution kernel. Note that only in the toric topology the kernel is consistent at the duplicated points. A similar point of view for toric images was suggested in \cite{maron2017convolutional}. We extend it to a more general family for toric images of sphere-type surfaces.

	\vspace{-3pt}
	\subsection{Branched covering maps}
	This section provides a brief introduction to branched covering maps (for more details see \cite{Hatcher:478079}). We start with a formal definition:
	
	\begin{definition}
	Let $X$ and $Y$ be topological spaces. A map $E : X \to Y$ is a \textbf{branched covering map} if every point $y \in Y$ except for a finite set of points $\{b_1, \ldots ,b_k\}$ has a neighborhood $U \subseteq Y$, such that $E^{-1}(U)$ is a disjoint union of homeomorphic \footnote{A homeomorphism is continuous map with a continuous inverse.} copies of $U$.
		
	The set of points $\{b_1, \ldots ,b_k\}$ are called \textbf{branch points}.
	\end{definition}
	
	A simple example for a branched covering map is ${E(z) = z^d}$, for $X=Y=\mathbb{C}$, and for some integer $d$. The function $E$ has one branch point at $b_1=0$. Every point $y\in Y \setminus \left \{0\right \}$, has $d$ \textit{distinct} pre-images ${E^{-1}(y) = \{x_1, \ldots, x_d\} \subset X}$. However, the point $y=0$ has a single pre-image $E^{-1}(y) =\{0\} $. We say that the point $y=0$ has $d$ pre-images located at $0$, or that $0$ is a pre-image with multiplicity $d$. The \textbf{ramification index} of $x$ over $E(x)$ is the multiplicity at $x$, namely $1$ for all $x\in E^{-1} ( Y \setminus \set{0})$ and $d$ for $x=E^{-1}(b_1)=0$. We denote it as $r(x|E(x))$. Figure \ref{fig:branched}b shows this example for $d=4$. In fact, this example captures all the local behaviors of covering maps: around a point $x \in X$ with $r(x|E(x)) = r$ the map $E$ looks like the map $z \mapsto z^r$.
	
	\begin{figure}[t!]
	\begin{tabular}{cc}
		\includegraphics[width=0.55\columnwidth]{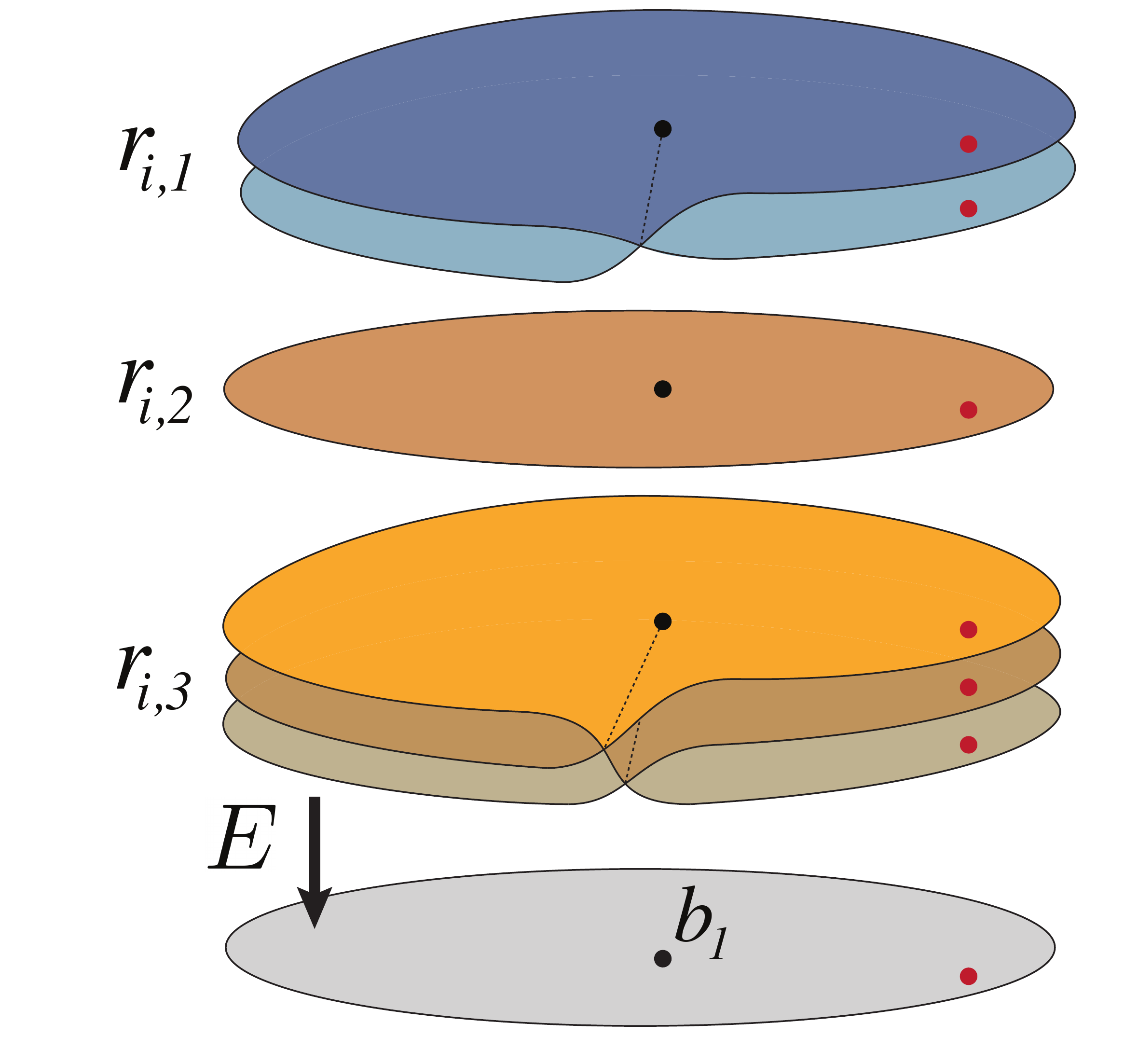} &
		\includegraphics[width=0.33\columnwidth]{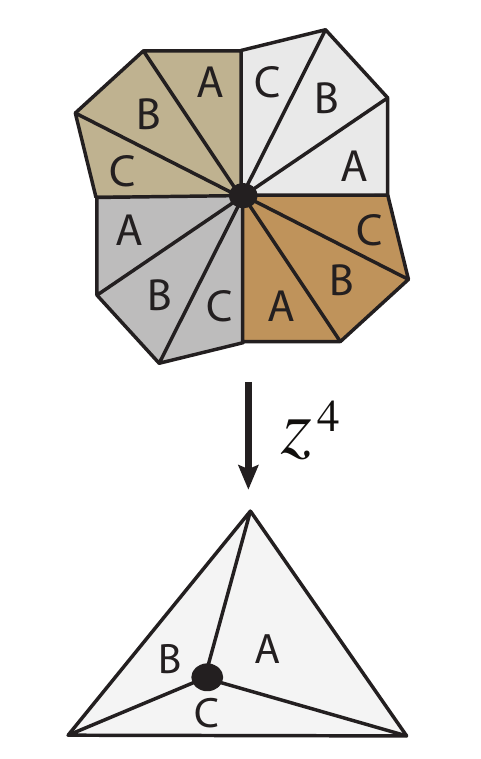}\\
		(a) & (b)	
	\end{tabular}
	\caption{(a) A branch point $b_1$ (in black) with ramification structure $\rho_1=[2, 1, 3]$. Note that all other points, such as the marked red point, have six distinct pre-images. (b) Shows an example of a branched covering map with degree $4$ on a triangular mesh; The central point has a single pre-image (above it, also in black).  The map looks like the map $z \mapsto z^4$. Each color represents a different copy of the neighborhood of the branch point. A triangle and its pre-images are marked with the same letter. } \vspace{-15pt}
	\label{fig:branched} 
	\end{figure}
	
	Let us give another example: Consider the function ${E(z) = z^2(z-7)}$. It has a branch point at $y=0$ with \textit{two} distinct pre-images. Namely, $E^{-1}(0) = \{0, 7\}$. Here, the ramification index of $0$ over $E(0)$ is $2$ and the ramification index of $7$ over $E(7)$ is $1$. We say that the \textit{ramification structure} of $0$ is $[2,1]$, formally:
	
	\begin{definition}
		Let $E:X\to Y $ be a branched covering map, $b_i$ a branch point and $l_i=|E^{-1}(b_i)|$, the number of pre-images of $b_i$. The \textbf{ramification structure} of $b_i$ is the multi-set of ramification indices of its pre-images, denoted by $\rho_i = [r_{i,1}, \ldots r_{i,l_i}]$. The \textbf{ramification type} of $E$ is the collection of its ramification structures, $\rho = [\rho_1, \ldots , \rho_k]$.
	\end{definition}
	
	Figure \ref{fig:branched}a depicts a branch point $b_1$ with three distinct pre-images, $l_1=3$, and ramification structure $\rho_{1} = [2, 1, 3]$. Note that the ramification structure of a non-branch point is a trivial multi-set of ones: $[1\dots 1]$, see \eg, the red dot in Figure \ref{fig:branched}a. 
	
	The sum of the ramification indices of any point in $X$ is independent of the choice of the point (see \cite{donaldson2011riemann}, page 44 Proposition 7), namely 
	\begin{equation}
	\sum_{j=1}^{l_i}r_{i,j}=d. 
	\label{eq:degree}
	\end{equation}
	Lastly, $d$ is called the \textbf{degree of the covering}. Intuitively, the degree of the covering counts how many times $X$ covers $Y$, or alternatively how many copies of $Y$ can be found in $X$.

 \vspace{-4pt}
	\subsection{Riemann-Hurwitz formula}
	
	A key fact about ramification types of branched covering maps between (boundaryless) surfaces is the \textbf{Riemann-Hurwitz formula} (RH), which connects the genus (\ie, number of handles) of the surfaces with the ramification type. In our case, we map a torus to a sphere-type surface and get the corresponding RH formula:
	\begin{equation} \label{eq:RH}
	\sum_{i=1}^{k}\sum_{j=1}^{l_i}(r_{i,j} - 1) = 2d	
	\end{equation}
	A quick derivation of this formula is given in Section~\ref{sec:rh_formula}.

	Therefore, the RH formula sets a necessary condition on the possible ramification types $\rho$ of such branched covering maps. For example, the ramification type $\brac{[2],[2],[2],[2]}$ satisfies the RH equations but the ramification type $\brac{[2],[2]}$ does not (in this case $d=2$, $k=2$, $l_i=1$), implying that there is no covering map with this ramification type. We note that Equations \eqref{eq:degree} and \eqref{eq:RH} are necessary but not sufficient conditions.

	\section{Approach}
	Our goal is transferring signals (\ie, functions) from a sphere-type  surface $M$ to the image domain $I$ (\ie, the flat torus: unit square $[0,1]^2$ with opposite ends identified). This is done by constructing a branched covering map 	\begin{equation}
	E: I\to M
	\label{eq:cover_from_flat_torus}
	\end{equation} 
	and pulling back the signals to the image using $E$. That is, given a signal $f:M\too\Real^n$ that we want to transfer, the value of a pixel $p\in I$ is set to $f(E(p))$. We represent the surface $M$ using a triangular mesh. 
	
	\begin{figure}
  \begin{center}
    \includegraphics[width=0.48\textwidth]{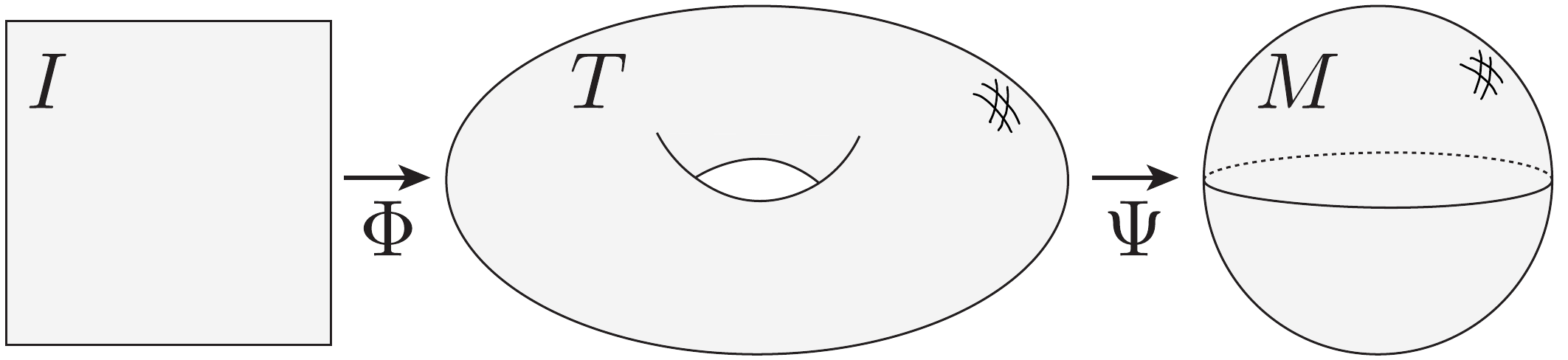}
  \end{center}
  \caption{Construction of the covering map $E:I\too M$.}\label{fig:E_construction}
\end{figure}
	We build the covering map $E:I\too M$ in two steps, as a composition of two functions:
	\begin{align*}
	E &=   \Psi \circ \Phi   \\
	\Phi&: I \to T  \\
	\Psi&: T \to M.
	\end{align*}
    where $T$ is a torus-type surface built out of $d$ copies of $M$, $\Psi$ is a branched covering map, and $\Phi$ is a homeomorphism between the two tori $I$ and $T$ (see Figure~\ref{fig:E_construction} for illustration).

	\subsection{Computing the branched covering map $\Psi$}
	
	In this section we describe how we construct the mesh $T$ out of the mesh $M$ and the branched covering map $\Psi:T \to M$. The idea is to cut and glue together several copies of the input surface $M$ in a way that generates a toric covering space corresponding to a specific choice of $\rho$.
	
	First, we choose $k$ branch points $b_1, \ldots , b_k$ from the set of vertices of $M$ (using farthest point sampling), a degree $d$ and a valid ramification type $\rho$ satisfying Equations (\ref{eq:degree})-(\ref{eq:RH}). Our algorithm then consists of the following steps:
	
	\textbf{Step (i):} We cut the mesh along $k$ disjoint paths, all emanating from the same (arbitrary) vertex $v_0$ in $M$ and ending at the branch points $b_i$ for $i\in [k]$. Figure \ref{fig:gluing} shows this for $k=d=5$. Topologically, $M_{disk}$ is a disk, with all branch points at its boundary. 
	
	\textbf{Step (ii):}
	$M_{disk}$ is then duplicated $d$ times, to form copies $M_{disk}^{\scriptscriptstyle (1)},\ldots, M_{disk}^{\scriptscriptstyle (d)}$. Figure \ref{fig:gluing} shows the $5$ copies with $v_0$ as a white dot and the branch points as colored dots.\\ 
	
	\begin{figure}[h!]
	\includegraphics[width=\columnwidth]{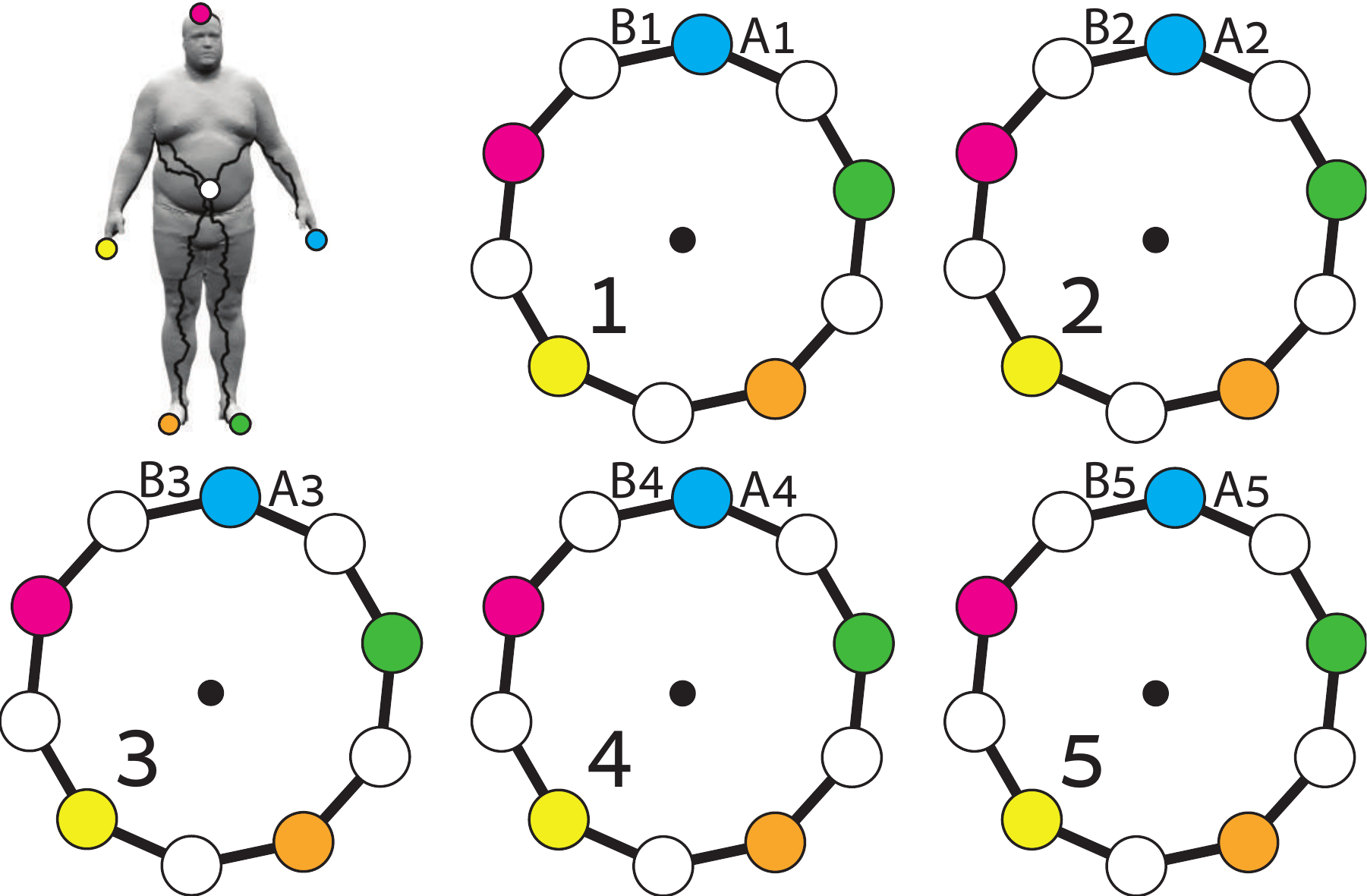}
	\caption{Illustration of the mesh cutting (top-left), and the gluing algorithm. Branch points are visualized as colored dots. }\label{fig:gluing}
	\end{figure}

	\textbf{Step (iii):} 
	We glue the $d$ copies of $M_{disk}$ to create the surface $T$ as follows. Consider a branch point $b_i$; it has $d$ copies located in each of the copies of $M_{disk}$, see \eg, the blue dots in Figure \ref{fig:gluing}. Denote by $B_j$ and $A_j$ the two boundary edges emanating from the $j$-th copy of $b_i$. Note that on the original surface $A_j$ is glued to $B_j$; since every $B_{j'}$ is a duplicate of $B_j$, $A_j$ can be glued to any $B_{j'}$, $j'\in[d]$. Therefore, to describe the gluing of the edges emanating from $b_i$ we use a  permutation $\sigma_i \in S_d$ (a permutation is a bijection $[d]\too [d]$): $A_j$ is glued to $B_{\sigma_i(j)}$. The collection of all permutations (one permutation per branch point)
	\begin{equation}\label{e:gluing_instruction}
	    \Sigma=\set{\sigma_1,\ldots,\sigma_k}
	\end{equation}
	is called the \textbf{gluing instructions}.  Given gluing instructions $\Sigma$ we use it to stitch the boundary of the $d$ copies of $M_{disk}$ to construct the toric surface $T$ (\ie, genus one). The mapping $\Psi:T\to M$ is then defined by: map $v\in T$ to its original version in $M$, and extend linearly in each triangle (\ie, face) of $T$.  $\Psi$ is a well defined branched covering map. The gluing procedure is summarized in Algorithm~\ref{alg:glue}. In Subsection \ref{ss:computing_Sigma} we describe the algorithm for computing the gluing instructions given the desired ramification type $\rho$.

	\begin{algorithm}[h]
		\KwData{cut mesh copies $M_{disk}^{\scriptscriptstyle (j)}$, $j\in [d]$\; \qquad \ \ \   gluing instructions $\Sigma=\set{\sigma_1,\ldots,\sigma_k}$ \\ 
		}
		\KwResult{The torus $T$ and a branched covering map $\Psi :T\to M$ with ramification type $\rho$}
		\For{every branch point $b_i$, $i\in [k]$}{
			\For{every copy $j\in [d]$}{
				stitch $A^{\scriptscriptstyle (i)}_j$ and $B^{\scriptscriptstyle (i)}_{\sigma_i(j)}$
			}
		}
		$\psi :T \to M$ is defined by mapping $v$ to the unique vertex in $M$ that originated $v$. 
		\caption{Gluing algorithm.} \label{alg:glue}
	\end{algorithm}
	
	\vspace{-20pt}	
	\subsubsection{Computing the gluing instructions}\label{ss:computing_Sigma}
  In this paper we limit our attention to ramification types of the form 
	\begin{equation}\label{e:rho_our}
	\rho = \brac{ [1^{d-r}, r ]^k },
	\end{equation}
	where $d$ is the cover degree, $k$ is the number of branch points, and $r$ is the maximal multiplicity of the branch points' pre-images. The motivation in choosing these ramification types is two-fold: First, we want all branch points to be treated equally by the cover. Second, applying higher ramification order improves area distortion of protruding parts (see \eg, \cite{kharevych2006discrete}); See Figure~\ref{fig:E0} and Subsection~\ref{s:example} for an example.
	
	First, let us compute necessary conditions for $\rho$ defined in  \eqref{e:rho_our} to be a feasible ramification type. Equation \eqref{eq:degree} is automatically satisfied since $d-r+r=d$. Plugging $\rho$ in \eqref{eq:RH} we get 
	\begin{equation}\label{eq:rhkind}
	k (r-1) = 2d.
	\end{equation}
	This sets a trade-off between $r$ and $d$: higher values of $r$, while reducing distortion of protruding parts would force higher degree $d$ of the cover, which will produce more copies of $M$ in the image. Practically, we found that $k=5,r=5,d=10$ and $k=6,r=2,d=3$ are both good options that strike a good balance between $r$ and $d$. 
	
    To compute gluing instructions $\Sigma$ we start with $k,r,d$ satisfying \eqref{eq:rhkind}. The next theorem (proved in Section \ref{s:proofs_and_gluing})%
     provides a necessary and sufficient condition for the gluing instructions $\Sigma$ to furnish a cover with ramification type $\rho$:
	\begin{theorem}
		\label{lem:nec}
		A set of gluing instructions $\Sigma=\set{\sigma_1,\ldots,\sigma_k}$ yields a branched covering map with ramification type $\rho$ if and only if the following conditions hold:
		\begin{enumerate}[(i)]
			\item The cycle structure of $\sigma_i$ equals the ramification structure of $b_i$, \ie, $\rho_i=[r_{i,1},\ldots, r_{i,l_i}]$.
			\item $\Sigma$ is a product one tuple. That is, $\sigma_1\cdot \sigma_2\cdots \sigma_k = I_d$.
			\item The group $G$ generated by $\Sigma$ is a transitive subgroup of $S_d$. Namely, for each $i,j\in [d]$ there exists $\sigma \in G$ so that $j=\sigma(i)$.  
		\end{enumerate}
	\end{theorem}
	
	Theorem \ref{lem:nec} indicates that we should search for permutations $\sigma_i$ with prescribed cycle structures. That is, the permutations $\sigma_i$, if exist, are in some prescribed conjugacy classes of the permutation group. Algorithm \ref{alg:tupleFinder} performs such a search, more or less exhaustively, using conditions (ii) and (iii) to prune options that cannot lead to a solution $\Sigma$. 
	
	Since theoretically not all $k,r,d$ satisfying Equation \eqref{eq:rhkind} have a corresponding covering map, Algorithm \ref{alg:tupleFinder} can terminate without finding gluing instructions. In this case, according to Theorem \ref{lem:nec} we know that there is no covering map with ramification type $\rho$. Nevertheless, it is rare to find such examples in practice and indeed we did not encounter such a case in our experiments. Table \ref{tab:product_ones} contains the results of Algorithm \ref{alg:tupleFinder} for any permissible $k,r,d$ with $k\le 6, d\le 10$ so that they can be used as input to Algorithm \ref{alg:glue}.

	\subsection{Flattening the toric surface }
	
	
	The last part of our covering map computation is the computation of the map $\Phi:I\too T$. Equivalently, we compute $\Phi^{-1}$.  To that end we use a version of the Orbifold-Tutte embedding \cite{aigerman2015orbifold}. We first cut $T$ along the two generating loops of the torus (using \cite{jin2018conformal}, Algorithm 5) to get a disk-type surface $T_{disk}$. Second, we compute a bijective piecewise affine map $\Phi^{-1}:T\too I$ by solving a sparse linear system of equations $Ax=b$, where $A\in \Real^{m \times m}$ and $x, b\in \Real ^{m\times 2 } $, and $m$ is the number of vertices in the disk-like mesh $T_{disk}$. This system is a discrete version of the Poisson equation \cite{lovasz2004discrete}, see Section \ref{s:orbitutte} 
	for details on how to construct $A,b$. We use $x$ to map the vertices of $T$ to $I$ and extend linearly to get the piecewise affine map $\Phi^{-1}$.
	
		The resulting map is discrete harmonic \cite{lovasz2004discrete}, approximately conformal up to a linear transformation, and as proven in \cite{aigerman2015orbifold}, a bijection.


	\subsection{Example}
	\label{s:example}
	Figure \ref{fig:E0}a depicts the case $k=5$, $r=3$, $d=5$. Thus $\rho =  \brac{[1,1,3]^5}$; every branch point $b_i$ has three distinct pre-images, where two have ramification one, and one with order-$3$ ramification. The gluing instructions in this case, computed using Algorithm \ref{alg:tupleFinder}, are:
	\begin{align*}
	    \Sigma = \{ &(1)(2)(345), (1)(4)(235),  (3)(4)(152), \\
	                &(3)(4)(125), (1)(5)(243) \} \ .
	\end{align*}
	Note that each of these permutations has a cycle structure $[1,1,3]$ as required in Theorem \ref{lem:nec} \textit{(i)}; conditions \textit{(ii)-(iii)} can be checked as well. These gluing instructions were used to glue the $5$ copies of $M_{disk}$ (as shown in Figure \ref{fig:gluing} and described in \Algref{alg:glue}) to generate the representation $E: I \to M$ shown in Figure \ref{fig:E0}a.
	
	\begin{figure}[t]
		\begin{center}
			\includegraphics[keepaspectratio, height=6cm, width=\linewidth ]{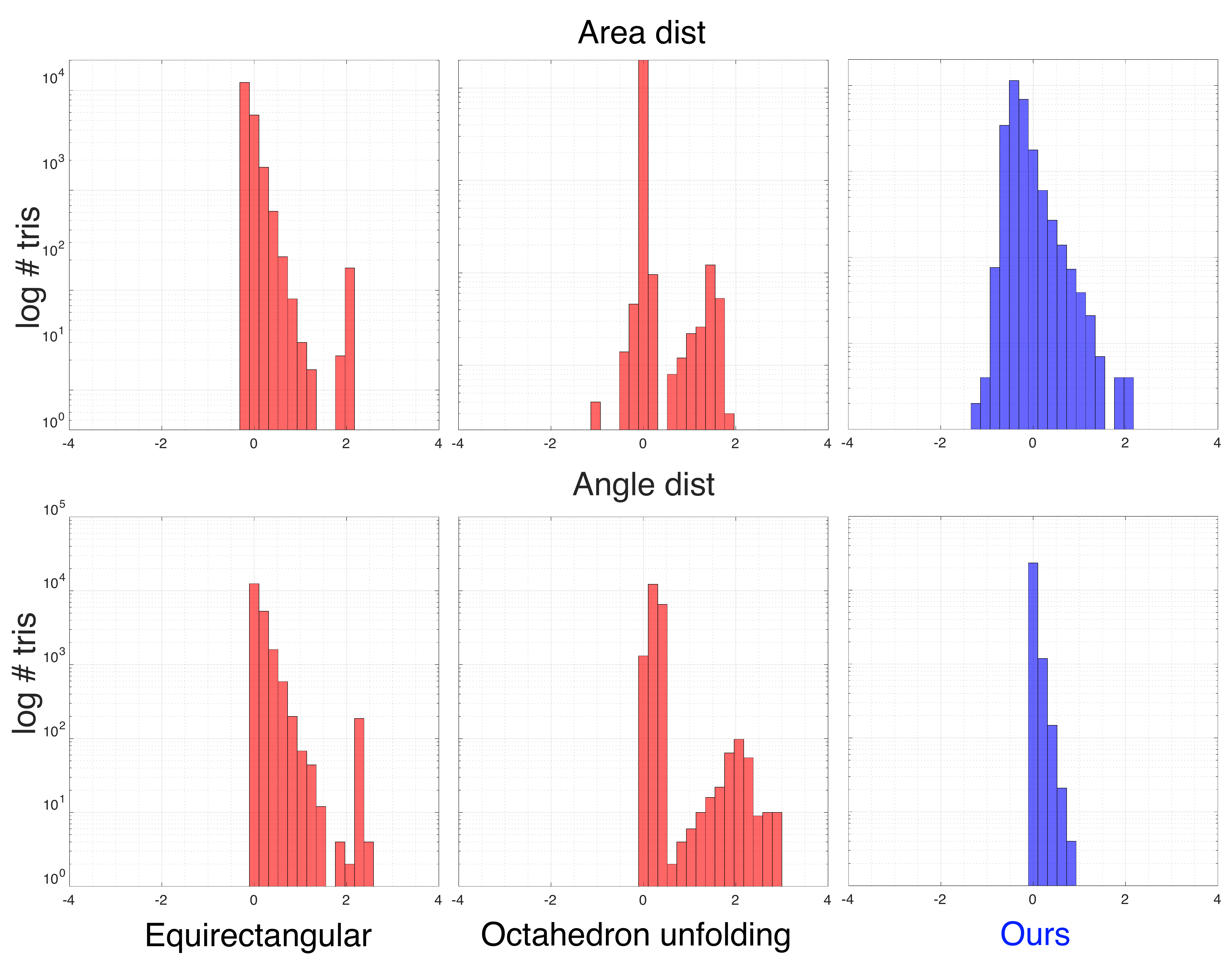}
		\end{center}
		\caption{Area and scale distortion (log scale) of our method (right, in  blue), equirectangular parametrization \cite{su2017learning} and parametrization by octahedron unfolding \cite{praun2003spherical}.	\vspace{-10pt}	}
		\label{fig:area_scale_sphere}
	\end{figure}

	\begin{table*}[t]
		\centering
		\caption{Comparison of our method and the top results in each category of the SHREC17 shape retrieval task.}
		\begin{tabular}{llllll}
			\multicolumn{1}{l}{Method} & \multicolumn{1}{l}{P@N} & \multicolumn{1}{l}{{R@N}} & \multicolumn{1}{l}{{F1@N}} & \multicolumn{1}{l}{mAP} & \multicolumn{1}{l}{NDCG} \\
			\midrule
			FURUYA\_DLAN & \textbf{0.814} & 0.683 & 0.706 & 0.656 & 0.754 \\
			Tatsuma\_ReVGG & 0.705 & \textbf{0.769} & 0.719 & 0.696 & 0.783 \\
			SHREC16-Bai\_GIFT & 0.678 & 0.667 & 0.661 & 0.607 & 0.735 \\
			Deng\_CM-VGG-6DB & 0.412 & 0.706 & 0.472 & 0.524 & 0.642 \\
			Spherical CNN \cite{cohen2018spherical} & 0.701 & 0.711 & 0.699 & 0.676 & 0.756 \\
			SO(3) Equivariant CNNs \cite{esteves2018learning}& 0.717 & 0.737 & - & 0.685 & -\\
			\midrule
			Ours  &0.749 ($2^{nd}$)       & 0.741 ($2^{nd}$)      &\textbf{0.734}      &\textbf{0.709}     &\textbf{0.794} \\
			\bottomrule
		\end{tabular}%
		\label{tab:SHREC_17l}
	\end{table*}%

\vspace{-4pt}
	\section{Experiments}\label{s:experiments}
	To evaluate the efficacy of our method we tested it in two main scenarios: learning signals on the sphere, and learning sphere-type surface data. 
	
	\vspace{-5pt}
	\subsection{Evaluation}
    In this section we compare the geometric properties of our representation to standard or existing techniques. Figure~\ref{fig:area_scale_sphere} shows the area and scale distortion of our method (right, in blue) and two other popular methods for sphere flattening:  Equirectangular projection (see \eg, \cite{su2017learning}) and octahedron unfolding projection, see \cite{praun2003spherical}.
    Area distortion is computed as the determinant of the differential of the cover map $E$ (treated as affine over each triangle of $M$), and angle distortion is the condition number of the differential. Since our image representation contains several copies of each triangle of $M$ we use the least distorted one for the histogram, as we want each part of the surface to appear in the image at-least once with low distortion. As can be seen in Figure~\ref{fig:area_scale_sphere}, our projection has better angle preservation with only a mild sacrifice to area distortion.

	In Figure \ref{fig:area_scale} we repeat this experiment with a sphere-type model of a human and compare the area and angle distortion of five different types of image representations. While the method of \cite{sinha2016deep} (leftmost, in red) preserves area better, it suffers from significant angle distortion. The orbifold covering of \cite{maron2017convolutional} (second to the left, in red) is angle-preserving, but suffers from notable area shrinking. Our covering maps (green and blue) strike a balance between angle and area preservation. The covering of type $[[1^5, 5]^5]$ (middle, in blue) has the least area distortion and we chose it for the segmentation task (below).
	
	The top row of Figure \ref{fig:area_scale} compares the different image representations by reconstructing the original model. Specifically, for each vertex of the mesh we sampled its $x,y,z$ coordinates directly from the image at the vertex location (we used $512\times 512$ images here).  In our representation, we take the coordinates from the vertex copy with the least area distortion. Note that the image representations of \cite{sinha2016deep} and \cite{maron2017convolutional} do not represent well significant parts of the surface (\eg, the right leg and the head). 

	\begin{figure}[t]
		\begin{center}
			\includegraphics[width=\linewidth]{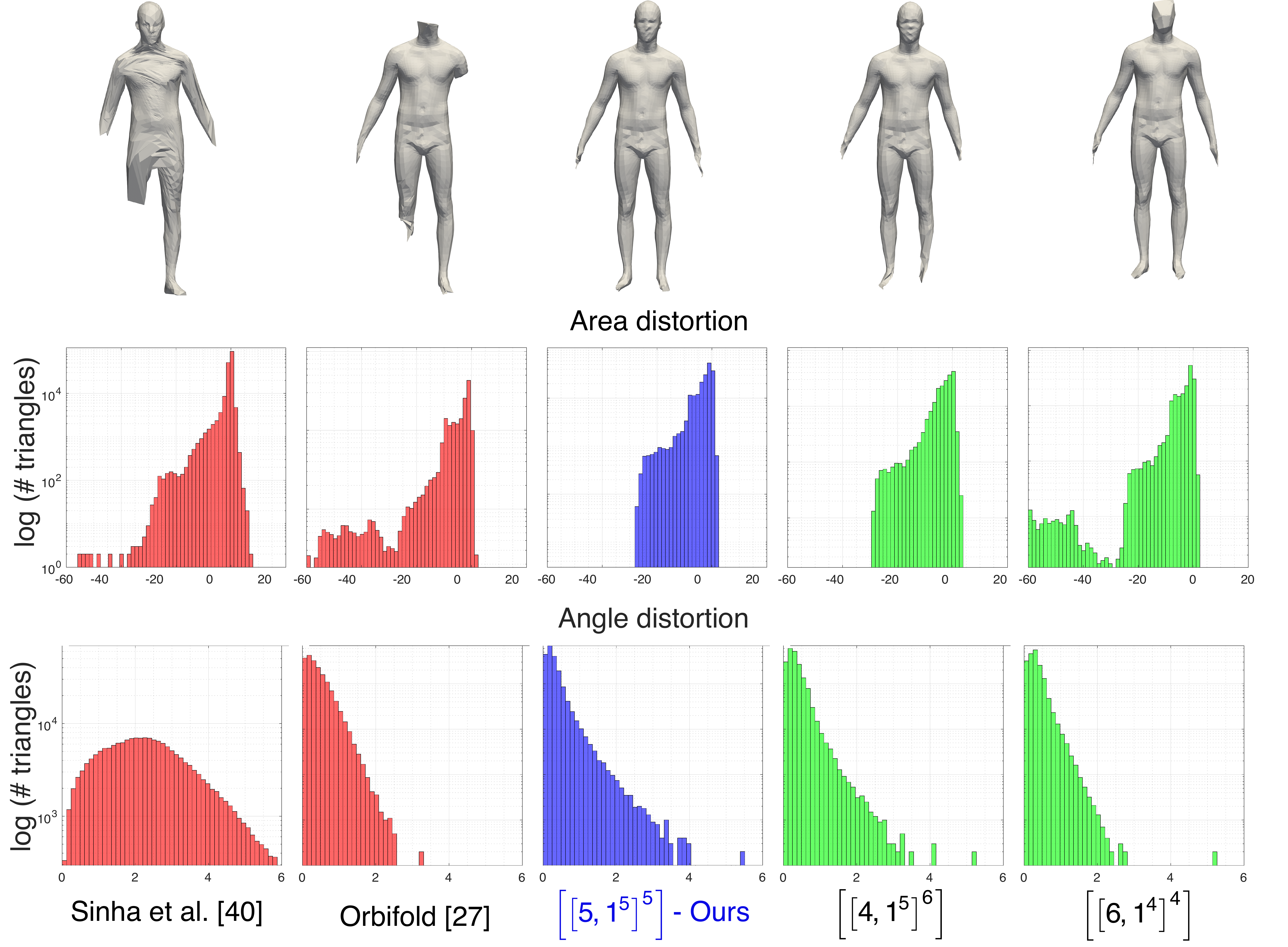}
		\end{center}
		\caption{Top row: meshes reproduced from image representations of several methods (left to right): area preserving method of \cite{sinha2016deep}, orbifold covering of \cite{maron2017convolutional} and three different covering maps produced by our method, of ramification types $[[1^5,5]^5], [[1^5,4]^6]$ and $[[1^4,6]^4]$. Middle and bottom rows: area and angle distortion (log scale). The histograms are taken from ten randomly chosen human models.}
		\label{fig:area_scale} \vspace{-15pt}
	\end{figure}

	\vspace{-5pt}

	\subsection{3D shape retrieval}	\vspace{-5pt}
	The first application of our method is 3D shape retrieval. We use the SHREC2017 benchmark \cite{savvashrec} that contains $51162$ 3D models from $55$ different categories. 
		There are two separate challenges: (i) the shapes are consistently aligned (ii) the shapes are randomly rotated. We tackle the (harder) second challenge.
	
	Since the shapes are not of genus zero we follow the protocol of \cite{cohen2018spherical} that project the meshes on a bounding sphere using ray casting, and record six functions on this sphere: distance to the model, $cos/sin$ of the model angles (this is done for both the model and its convex-hull). We then use our method to transfer these six spherical signals to periodic images (flat torus). See Figure \ref{fig:guit} for an example of such shape representation.    
	
	We compare our method to the top methods in each category, the Spherical CNN method \cite{cohen2018spherical}, and the recent SO-3 equivariant networks suggested in \cite{esteves2018learning}. The results are summarized in Table \ref{tab:SHREC_17l}; note that in the F1 measure we score first among all methods.   
	
	For this application we use a slight modification of the inception v3 architecture \cite{szegedy2016rethinking}. We train the network with ADAM optimizer \cite{kingma2014adam} for $100$ epochs with learning rate $0.05$, batch size of $32$, and learning rate decay of $0.995$. Training took $15$ minutes per epoch on a Tesla V100 Nvidia GPU. In evaluation time we average the output of the network on $5$ randomly rotated copies of the query model.  
		\vspace{-3pt}
	\subsection{Surface classification}	\vspace{-3.5pt}
		\begin{figure}[t]
		\includegraphics[width=\linewidth, scale =0.3]{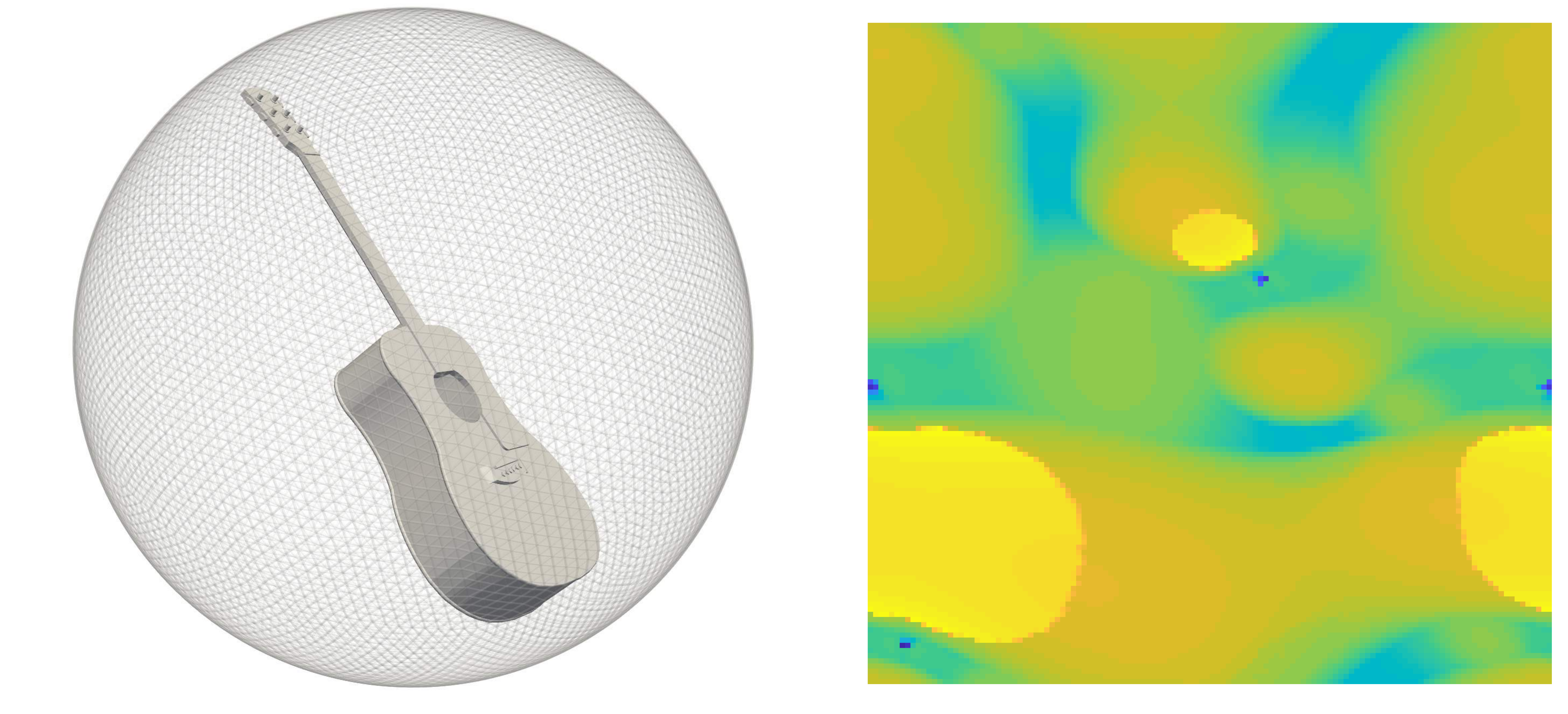}
		\caption{Left: A guitar mesh inside an encapsulating sphere. Right: The image representation of a spherical signal for this mesh. The signal is generated as follows: for each point on the sphere we cast a ray towards the origin and record the ray length at the intersection with the mesh. \vspace{-16pt}}
		\label{fig:guit}
	\end{figure}
	We apply our method to the ModelNet40 surface classification benchmark  \cite{Wu_2015_CVPR} that contains $12311$ 3D models from 40 different categories.
		As in the shape retrieval task, we follow the protocol of \cite{cohen2018spherical} to generate input signals on a sphere. We then use our method to represent the spherical signals as periodic images and apply the same inception v3 model as in the shape retrieval task.  We present peak performance results (following \cite{jiang2019spherical}) for two scenarios that are popular in the literature: (i) the shapes are rotated randomly about the $z$ axis; and (ii) the shapes are learned in their original orientation. We train the network with ADAM optimizer \cite{kingma2014adam} for $100$ epochs for scenario (1) and $300$ for scenario (ii) with learning rate $0.0005$, batch size $16$, and learning rate decay $0.995$. Training took $19$ minutes per epoch for the first scenario (that contains $10$ rotation augmentations) and $3$ minutes per epoch for the second scenario on a Tesla V100 Nvidia GPU.
		
	Table \ref{tab:cls_res} compares our results with several recent methods including the baselines of equirectangular projection (\eg, \cite{su2017learning}) and octahedron unfolding projection \cite{praun2003spherical}. Our results are the best among all spherical learning methods.
	
	\begin{table}[t]
	\centering
	\caption{Results on ModelNet40 dataset.}
	\begin{tabular}{llr}
		Method & Inputs & \multicolumn{1}{l}{Accuracy}\\
		\midrule
		Learning Gims \cite{sinha2016deep}& mesh & 83.9\%\\
		3DShapeNets \cite{Wu_2015_CVPR} & voxels & $84.7\%$ \\
		VoxNet \cite{maturana2015voxnet}& voxels & $85.9\%$ \\
		Pointnet\cite{qi2017pointnet}& points & $89.2\%$ \\
		Pointnet++ \cite{qi2017pointnet++}& points&$91.9\%$ \\
		Dynamic graph CNN \cite{wang2018dynamic}& points & $92.2\%$ \\
		PCNN \cite{atzmon2018point} & points& $92.3\%$ \\ 
		\midrule
		Spherical CNN \cite{cohen2018spherical}& spherical  & $85.0\%$ \\
		SO(3) Equivariant CNNs \cite{esteves2018learning} &	spherical& $88.9\%$ \\
		Spherical on unstructured grid \cite{jiang2019spherical}&	spherical	& $90.5\%$ \\
		Octahedron unfolding (rot $z$) & spherical & $90.2\%$ \\
		Equirectangular projection (rot $z$) & spherical & $90.1\%$ \\		
		\textbf{Ours}& \textbf{spherical} & $\mathbf{91.6\%}$ \\
		\textbf{Ours (rot $z$)}& \textbf{spherical} & $\mathbf{91.0\%}$ \\
		\bottomrule
	\end{tabular}%
	\label{tab:cls_res}
	\end{table}%
	
	\vspace{-5pt}
	\subsection{Surface segmentation} 	\vspace{-3pt} While our first two application targeted spherical signals, our last applications learns signals defined on general sphere-type human models. In particular, we perform human model semantic segmentation. We use the benchmark from \cite{maron2017convolutional} that consists of 373 train models from multiple sources and 18 test models. $5\%$ randomly sampled train models were used as a validation set (18 models). All models are given as triangular meshes. For each model, each face is labeled according to a predefined partition of the human body (\eg, head, torso, hands, total of $8$ labels). The task is to label the triangles of a new unseen human model with these labels. For each model we generate an augmented set of $120$ images per mesh, by permuting the order of the branch points, multiplying the vertices by a random orthogonal matrix and a uniform scale sampled from $[0.85,1.15]$ as suggested in \cite{poulenard2018multi}, and small periodic image translations of $\pm 15$ pixels. In evaluation, as the toric image contains $d$ values for each triangle on the original mesh, we use the label of the triangle with the largest area. Furthermore, we use 10 random augmentations of test images and label each mesh face using a majority vote. 
	Table \ref{tab:seg_res} summarizes the results of this experiment, where our method outperformed previous methods; Figure \ref{fig:segs} shows typical segmentation results.
	
	For this application we used the U-net architecture \cite{ronneberger2015u} with $16$ layers (see Table \ref{tab:unet_arch} for details).  We used a weighted loss with equal probability labels, and trained the network using stochastic gradient descent with momentum \cite{sutskever2013importance} for $50$ epochs with learning rate $0.2$, batch size $2$, and learning rate decay of $0.995$. Training takes $\sim 3$ hours per epoch on a Tesla V100 Nvidia GPU.
	\begin{figure}[t]
		\includegraphics[width=\linewidth]{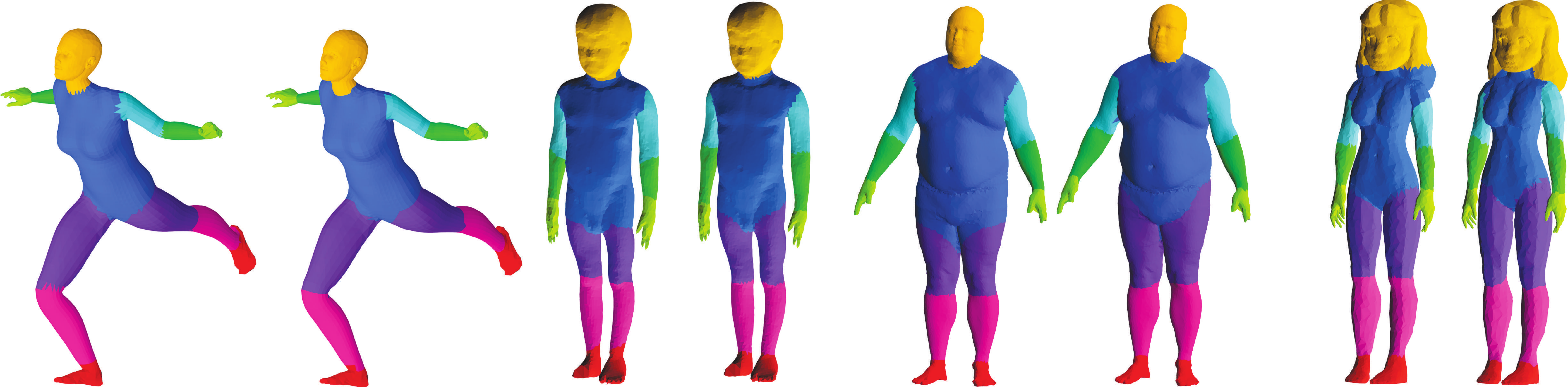}
		\caption{Segmentation results of our method. For each pair: left is our result and right is ground truth. \vspace{-5pt}}\label{fig:segs}
	\end{figure}
	
	\begin{table}[t]
		\centering
		\caption{Results on the human segmentation dataset. }
		\begin{tabular}{llr}
			Method & Inputs & \multicolumn{1}{l}{Accuracy} \\
			\midrule
			Toric CNN \cite{maron2017convolutional}& WKS,AGD,curv & $88.00\%$ \\
			Geodesic Conv \cite{masci2015geodesic} & 3D coords & $76.49\%$ \\
			Pointnet++ \cite{qi2017pointnet++}& 3D coords & $90.77\%$ \\
			Dynamic graph CNN \cite{wang2018dynamic}& 3D coords & $89.72\%$ \\
			Multi-directional Conv \cite{poulenard2018multi}& 3D coords & $88.61\%$ \\
			Learning Gims \cite{sinha2016deep}	&	3D coords & $84.53\%$ \\
			\textbf{Ours}& \textbf{3D coords} & $\mathbf{91.31\%}$ \\
			\bottomrule
		\end{tabular}%
		\label{tab:seg_res}
	\end{table}%

	\vspace{-10pt}
	\section{Conclusions}\vspace{-5pt}
	In this paper, we introduce a new method for representing sphere-type surfaces as toric images that can be used in standard Convolutional Neural Network frameworks for shape learning tasks. The method allows faithful representation of \emph{all parts of the surface in a single image}, thus alleviating the need to generate multiple maps to cover each surface. Our method is general and can target both spherical signal learning tasks as well as more general learning tasks that involve signals on different genus zero surfaces.  Practically, we showed that off-the-shelf CNN models applied to images generated with our method lead to state of the art performance in the tasks of shape retrieval, shape classification and surface segmentation. 
	
	The main limitation of this work is its restriction to genus-zero surfaces. This kind of models are abundant, but certainly do not exhaust all 3D models. We would like to seek a generalization of this method to point clouds, depth images and more general topological types. 
	
	
	\vspace{-8pt}
	\section{Acknowledgements} \vspace{-5pt}
	This research was supported in part by the European Research Council (ERC Consolidator Grant, ”LiftMatch” 771136) and the Israel Science Foundation (Grant No. 1830/17).
	
	{\small
		\bibliographystyle{ieee_fullname}
		\bibliography{final3}
	}

\appendix

\begin{figure*}
	    \includegraphics[width=\textwidth]{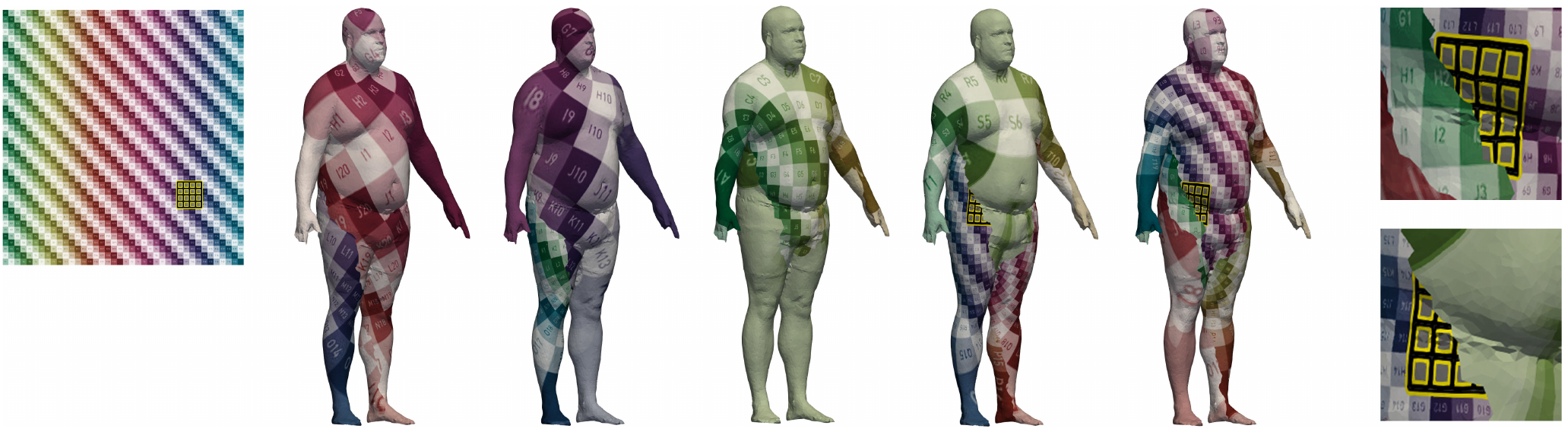}
		\caption{A topological torus (textured image, left) covers a human model $5$ times (middle). This cover is constructed out of 5 copies of the model stitched to form a topological torus (middle). The map from the flat torus (left) to the cover is visualized using a colored checkerboard. Since the 5 copies form a torus, up/down/right/left translations are well-defined everywhere on the cover. The standard image convolution is invariant to these translations over the surface (see right insets where the convolution kernel moves seamlessly between copies of the model).}
			\label{fig:cover}

\end{figure*}

\section{Convolution on a spherical mesh}
In Figure \ref{fig:cover} we depict a cover map from the torus (texture square image on the left) to a human surface (middle); this map covers the human $5$ times. We further show how standard convolution stencil (in yellow) translates to a seamless convolution on the surface. Note that the texture seams on the human models are pretty arbitrary and just indicate when moving to a different copy of the surface.

\section{Guidelines on Choosing parameters}
Adding branch points helps reducing the local distortion in protruding parts, therefore we recommend to choose as many branch points as there are protruding parts common in the dataset (\eg $5$ for humans, $8$ for octopuses etc.).
As we mentioned in section \ref{ss:computing_Sigma} we choose a ramification type of the form $[ [r , 1^{d-r} ] ]$ for each branch point.

As noted in Section~\ref{ss:computing_Sigma}, higher ramification ($r$) also improves area distortion of protruding parts. However, in that case, we are limited by the RH formula (\Eqref{eq:rhkind}). So we would recommend choosing the highest $r$ possible (\eg as appears in Table \ref{tab:product_ones}) and taking $d$ (number of copies) to satisfy \Eqref{eq:rhkind}. Also note that higher $r$ implies higher $d$ (number of copies). Therefore, for a fixed image resolution we would like the highest number of branch points for which all relevant parts are still visible in the image.

\section{Gluing Instructions}
As mentioned in Section \ref{ss:computing_Sigma}, for each choice of number of branch points $k$ , degree $d$  and ramification type $\rho$ satisfying Equations \eqref{e:rho_our} and  \eqref{eq:rhkind} We need to compute a product one tuple of permutations satisfying the conditions of theorem \ref{lem:nec}. We note that this computation can be done in an offline step, before using Algorithm \ref{alg:glue} to compute the toric parameterization. 
In Table \ref{tab:product_ones} We provide gluing instructions corresponding to each valid choice of $k\le 6$, $ d\le 10$ and $\rho$ that complies with Equations  \eqref{e:rho_our} and  \eqref{eq:rhkind}. Each of the gluing instructions in Table \ref{tab:product_ones}
can be used as input to Algorithm \eqref{alg:glue}.
 \begin{table*}[t]
     \centering
	\renewcommand{\arraystretch}{1.5}
     \begin{tabularx}{\textwidth}{lll|l}
     	\toprule
     $k$ & $d$ & $\rho$ & Gluing instructions \\ 
     \midrule
     3 & 3 & $\left[\left[3\right]^{3}\right]$ & $\left(1,2,3\right),\left(1,2,3\right),\left(1,2,3\right)$ \\
     3 & 6 & $\left[\left[1,5\right]^{3}\right]$ & $(1,2,3,4,5),(1,3,4,6,2),(2,6,3,5,4)$ \\
      3 & 9 & $\left[\left[1^{2},7\right]^{3}\right]$ & $(1,2,3,4,5,6,7),(1,7,6,2,3,8,9),(1,9,8,2,5,4,3)$ \\
      4 & 2 & $\left[\left[2\right]^{4}\right]$ & $\left(1,2\right),\left(1,2\right),\left(1,2\right),\left(1,2\right)$ \\
      4 & 4 & $\left[\left[1,3\right]^{4}\right]$ & $(1,2,3),(2,3,4),(1,2,4),(1,2,3)$ \\ 
      4 & 6 & $\left[\left[1^{2},4\right]^{4}\right]$ & $(1,2,3,4),(2,5,4,3),(1,5,6,4),(1,5,4,6)$ \\
      4 & 8 & $\left[\left[1^{3},5\right]^{4}\right]$ & $(1,2,3,4,5),(3,6,8,5,4),(1,5,4,7,2),(2,7,4,8,6)$ \\
      4 & 10 & $\left[\left[1^{4},6\right]^{4}\right]$ & $(1,2,3,4,5,6),(1,7,8,3,5,9),(2,10,8,7,6,5),(1,9,4,3,8,10)$ \\
      5 & 5 & $\left[\left[1^{2},3\right]^{5}\right]$ & $\left(3,4,5\right),\left(2,3,5\right),\left(1,5,2\right),\left(1,2,5\right),\left(2,4,3\right)$ \\
      5 & 10 & $\left[\left[1^{5},5\right]^{5}\right]$ & $\left(6,7,8,9,10\right),\left(1,7,3,4,9\right),\left(1,8,4,3,7\right),\left(2,5,4,7,6\right),\left(2,10,9,4,5\right)$ \\
      6 & 6 & $\left[\left[1^{3},3\right]^{6}\right]$ & $(1,2,3),(2,5,3),(3,6,5),(3,5,6),(1,4,5),(3,5,4)$ \\
      6 & 9 & $\left[\left[1^{5},4\right]^{6}\right]$ & $(1,9,3,5),(1,7,8,4),(3,7,5,6),(4,8,7,9),(1,3,6,2)$ \\
      \bottomrule
     \end{tabularx}
     \caption{Gluing instructions for choices of $k,d,\rho$}
     \label{tab:product_ones}
 \end{table*}

\section{Implementation Details}
\label{sec:implementations}

\paragraph{Learning.} We use Pytorch \cite{paszke2017automatic} for learning. All the experiments are done with toric images generated by our algorithm and off-the-shelf CNN architectures with a single change: we replace the standard zero padding with periodic padding.  

\paragraph{Data generation.} For the surface segmentation task we use a cover of the type ${\rho =\brac{[1^5,5]^5}}$, that is, ${\rho_i=[1^5,5], i\in [5]}$. For the spherical learning tasks (shape retrieval and classification) we use a cover of type ${\rho_i=[1,2], i\in [6]}$. The locations of the branch points are chosen using farthest points sampling. We use the shortest paths from an arbitrary base point to all branch points in order to cut the mesh. When the mesh does not allow such a path we subdivide it locally (without changing its geometry). This pre-processing step is implemented in Matlab. It takes $\sim22$ seconds in average (relatively long running time due to a non-optimized mesh cutting code in Matlab) to generate a periodic (toric) image for a mesh with $6890$ vertices on a single CPU core in an Intel(R) Xeon(R) CPU E5-2670 v3 @ 2.30GHz machine.

\subsection{Segmentation Task}

\paragraph{Prediction.} The network outputs per-pixel labels. In order to obtain a label for each face in the original mesh $M$, we first transfer the per-pixel logits to the faces $F_T$ of the toric mesh using bilinear interpolation sampled at the faces' centers. Since each face $f$ in $M$ has $d$ duplicated faces in the toric mesh $T$ ($\vert \Psi^{-1}(f) \vert = d$), each face $f$ in $M$ has $d$ sets of logits. We use a weighted average of the $d$ sets of logits, where the weights are the area scales of the faces $\Psi^{-1}(f)$. The label of $f$ is the argmax of this weighted-average of logits. This means that better scaled faces (in the toric mesh) receive more weight when deciding how to label a face in the original mesh $M$.

\paragraph{Architecture.} We use a version of a U-Net \cite{ronneberger2015u}. The feature-channels sizes are given in Table \ref{tab:unet_arch}. After each convolution we use $\mathrm{ReLU}$ with a Batch-Normalization layer \cite{ioffe2015batch}. Each UpSample layer is a nearest-neighbour interpolation with scale-factor 2.

%

\section{Proofs}
\subsection{Riemann-Hurwitz formula}
\label{sec:rh_formula}
Consider a branched covering map $E : T  \to M$ of degree $d$ and $k$ branch points, from a toric mesh $T= (V_T, E_T, F_T)$ to a spherical mesh $M = (V_M, E_M, F_M)$. We prove that the ramification type of $E$ must satisfy the Riemann-Hurwitz formula \eqref{eq:RH-ap}.

\begin{proof}[Proof of Riemann-Hurwitz formula]
	First, we note that the set of branch points $B = \{b_1, \ldots b_k\}$ can always be chosen from $V_M$.
	
	Every node $v\in V_M\setminus B$ has $d$ pre-images in $V_T$. However, a branch point $b_i$ has $l_i < d$ pre-images in $V_T$. Every edge $e \in E_M$ has exactly $d$ pre-images in $E_T$, that is ${|E_T|=d|E_M|}$. Similarly, ${|F_T|=d|F_M|}$.
	
	By computing the Euler characteristic for a toric surface:
	\begin{equation}
	\begin{split}
	0 &= \chi(T) = |V_T| - |E_T| + |F_T|  \\
	& = d \underbrace{(|V_M| - |E_M| + |F_M|)}_{\chi(M)} - \sum_{i=1}^{k}(d-l_i) \\
	& = 2d - \sum_{i=1}^{k}(d-l_i) 
	\label{eq:RHformulation}
	\end{split}
	\end{equation}   
	
	Using \begin{equation}
	\sum_{j=1}^{l_i}r_{i,j}=d    
	\end{equation}
	
	and rewriting we obtain the Riemann-Hurwitz formula (RH), in its version for a map from a toric surface to a spherical surface:
	\begin{equation}
	\sum_{i=1}^{k}\sum_{j=1}^{l_i}(r_{i,j} - 1) = 2d	
	\label{eq:RH-ap}
	\end{equation}
\end{proof}

\subsection{Proof of Theorem $\ref{lem:nec}$}
\label{s:proofs_and_gluing}

We recall the following topological facts. A degree $d$ branched covering map $E:T \to M$ from a torus to a genus $0$ surface induces a group homomorphism, called the monodromy representation, from $\pi_{1} ,$ the fundamental group of 
$M\setminus\left\{ b_{1},\ldots,b_{k}\right\}  $ to $S_d$.

The homomorphism is given as follows: We take each loop $l \in \pi_{1},$ based at a point $p$, and lift it to $T$ starting from a preimage of $p$. This lift has to end at another preimage of $p$. Due to properties of the lifting, this induces a permutation on the preimages of $p$ in $T$, referred to as the \emph{fiber} of $p$. 

The group $\pi_{1}$ has $k$ generators and a single relation. The generators, $l_{1}\ldots,l_{k},$ are the $k$ loops around each of the branch points. The relation is $l_{1}*\ldots* l_{k} = 1$.

Our gluing instructions, $\sigma_{1},\ldots,\sigma_{k} $, will be the images of $l_{1},\ldots,l_{k}$ under the monodromy representation.
We shall now give a proof of Theorem \ref{lem:nec} . Namely, that our algorithm produces a cover $T \to M$ with ramification $\rho$ if and only if the gluing instructions are a tuple of permutations satisfying the conditions of Theorem \ref{lem:nec}.
\begin{proof}[Proof of Theorem \ref{lem:nec}]
	First we prove that the conditions in the theorem are necessary.
	
	For $(i)$, we note that a lift of a loop around a branch point $l_{i}$ with a particular ramification structure induces a permutation with the same cycle structure.\\
	For $(ii)$, the fact that $l_{1}*\ldots *l_{k} = 1,$ implies (using group homomorphism)  that $\sigma_{1}\cdot\ldots\cdot \sigma_{k} = I_{d}.$ \\
	For $ (iii),$ fix $p_{1} ,p_{2}$ in the fiber of $p.$ Since $ T$ is connected, there exists a path $\gamma$ connecting $p_1$ and $p_2$. The loop $E\circ \gamma$ is a loop starting and ending at $p$ whose lift takes $p_1$ to $p_2$. Thus, the action of group generated by $\Sigma = \set{\sigma_{1},\ldots,\sigma_{n}}$ is transitive.
	
	Conversely, suppose we have a product one tuple $ \sigma_1, \ldots , \sigma_{k}$ satisfying the conditions of the theorem and $k$ branch points $b_1, \ldots, b_k$. Then condition (i) allows us to define an action of the group $H:= \left\langle \sigma_{1},\sigma_{2},\ldots,\sigma_{k}\right\rangle $ on $[d]$. Following the construction in \cite{Hatcher:478079} pg 68-70 the space $\nicefrac{U \times [d]}{\pi_{1} \times H } $ is a covering space of $M$, where $U$ is the universal cover of $M$. The transitivity of $H$ implies that this covering space $C$ is connected. Condition $(iii)$ implies by the Riemann-Hurwitz formula that $C$ is topologically a torus.
	
	Let $D$ be the space produced from Algorithm \ref{alg:glue}. Note that the construction in Algorithm \ref{alg:glue} implies that lifting a loop circling each branch point $b_i$ induces the permutation $\sigma_{i}$ on the fiber of a generic point. Thus, the action of $\pi_{1}$ on $D$ coincides with the action of $\pi_{1}$ on $C$. Since every action of $\pi_{1}$ on $[d]$ (up to conjugation) produces a unique (up to homeomorphism) covering space, we deduce that $D$ is homeomorphic to $C$.

\end{proof}

\paragraph{Comment:} The equivalence between branched covering maps and tuples of permutations satisfying the conditions of Theorem \ref{lem:nec} is well known. This equivalence is commonly referred to as Riemann's existence theorem (RET). However, to the best of our knowledge, it was previously not known how to practically construct any given branched covering map (our Algorithm \ref{alg:glue}). 

\section{Gluing Instructions}
We now turn to describing an algorithm that finds tuples of permutations
$\sigma_{1},\ldots,\sigma_{k}\in S_{d},$ corresponding to a prescribed
ramification structure $\rho $, up to simultaneous conjugation (relabeling of the branch points). We call such a tuple a product one tuple. We implement our algorithm using Magma computational algebra system \cite{magmamagma}.

We denote the conjugacy class in $S_{d}$ associated with the cycle structure of $\rho_{i}$ by $C_{i}$. In the algorithm construction we use the following:
\begin{claim}
	\begin{enumerate}
		\item $\left\langle \sigma_{1},\sigma_{2},\ldots,\sigma_{k-1}\right\rangle $
		is a transitive permutation group and $\Pi_{i=1}^{k-1}\sigma_{i} \in C_n ,$ if and only if $\sigma_{1},\sigma_{2},\ldots, \sigma_{k}$, where $\sigma_{k} = \left (\sigma_{1} \sigma_{2} \cdots \sigma_{k-1} \right)^{-1}$ is a transitive product one tuple with $\sigma_k\in C_k$.
		\item The set $\set{\sigma_1,\ldots,\sigma_i}$ can be completed to a transitive product one tuple compatible with a ramification structure $\rho$ if and only if $\set{\sigma_1,\ldots, \sigma_{i-1}, g \sigma_i g^{-1}}$, for any $g\in Z(\sigma_1,\ldots, \sigma_{i-1})$ ($Z$ denotes the centralizer), can be completed to a transitive product one tuple compatible with $\rho$. 
	\end{enumerate}
\end{claim}
\begin{proof}
	(1) follows from the observations that adding elements to a transitive generator set keeps the set transitive, and that for $g \in S_d$ the cycle structure of $g$ and $g^{-1}$ are the same. For (2), note that for any $g\in Z\left(\sigma_{1},\ldots,\sigma_{i-1}\right)$ and $j\in\left[i-1\right]$ it holds that $g\sigma_{j}g^{-1}=\sigma_{j}$. Thus,  for any $g\in Z\left(\sigma_{1},\ldots,\sigma_{i-1}\right)$, we have that any tuple with $\sigma_{1}, \ldots, \sigma_i$ is the same as a tuple with $g\sigma_{1}g^{-1},\ldots ,g\sigma_{i}g^{-1}$,  up to simultaneous conjugation.
\end{proof}

The main idea in the algorithm for finding all gluing instructions corresponding to a ramification type $\rho$ is to exhaustively go over all tuples $\sigma_{i} \in C_{i}$ and check whether they form a product one tuple. We use the claim above to prune this exhaustive search, as described in Algorithm \ref{alg:tupleFinder}. Note that this computation is done once for a given cover ramification type and is reused for all models using this type of cover.


\begin{algorithm}[h!]
	
	\KwData{ a ramification structure $\rho=(\rho_1,\ldots,\rho_k)$ }
	\KwResult{all gluing instructions $\Sigma$ compatible with $\rho$ }
	Pick $ \sigma_1\in C_1$ \\
	Call the recursive function \textbf{tuplesFinder}($\rho$,$\sigma_1$) \\ \vspace{10pt}
	\textbf{tuplesFinder}($\rho$, $\set{\sigma_1,\ldots,\sigma_i}$)\\ 
	\While{ $C_{i} \ne \emptyset $}{
		pick $\sigma _i\in C_i$ \\
		update $ C_i = C_i \setminus Z(\sigma_{1},\sigma_{2},\ldots,\sigma_{i-1})  \sigma_i Z(\sigma_{1},\sigma_{2},\ldots,\sigma_{i-1}) ^{-1} $ \\
		\eIf  {$i = n-1$}{
			\If {$\Pi_{i=1}^{n-1} \sigma_{i}$ and $\sigma_n$ are conjugates and $\left\langle \sigma_{1},\sigma_{2},\ldots,\sigma_{n-1}\right\rangle$ is transitive }{
				add $\sigma_{1},\sigma_{2},\ldots,\sigma_{k}$ to list of product one tuples
			}
		}{
			call \textbf{tuplesFinder}($\rho$, $\set{\sigma_{1},\sigma_{2},\ldots,\sigma_{i}}$ )
		}		
	}
	\caption{Finding gluing instructions.}  \label{alg:tupleFinder}
\end{algorithm}

\section{Orbifold-Tutte embedding of $T$}
\label{s:orbitutte}
We compute $x$ by solving a sparse linear system following \cite{aigerman2015orbifold}:
\begin{equation}
Ax= b
\label{eq:linear_system}
\end{equation} Here $A\in \Real^{m \times m}$ and $x, b\in \Real ^{m\times 2 } $, where $m$ is the number of vertices in the disk-like mesh $T_{disk}$. The linear system \eqref{eq:linear_system} is constructed by putting together four sets of linear equations as follows:

First, for all interior vertices we set the discrete harmonic equation: \begin{equation}
\sum_{u\in N_{v}} w_{vu}\left( x_v - x_u\right)=0
\end{equation}
where $N_{v}$ is the set of vertices in $V_{T_{disk}}$ adjacent to $v$ and $w_{uv}$ are the cotangent weights \cite{pinkall1993computing}.

Let $L_1$ and $L_2$ be the generators of the homotopy group of $T$. Denote by $v_0\in V_{T}$ the intersection of the two loops $L_1$ and $L_2$ . In $T_{disk}$, the vertex $v_0$ has four copies $v_1', v_2', v_3', v_4'$. next, we ensure that these four copies are mapped to the four corners of the unit square $[0, 1]^2$. Explicitly,  \begin{equation}
v_1' = [0, 0]^T, v_2' = [1, 0]^T, v_3' = [1, 1]^T, v_4' = [1, 0]^T
\end{equation}
Each vertex $v \in \partial V_{T_{disk}} \setminus{\{v_1', v_2', v_3', v_4'\}}$ has a twin vertex $\tilde{v}$ such that $v$ and $\tilde{v}$  correspond to the same vertex in the uncut mesh $T$. Moreover, each such vertex $v$ has its origin in $V_{T}$ either in $L_1$ or in $L_2$. 

We set the vertices whose origin is in $L_1$ to be different by a constant translation in $[0, 1]^T$ and the vertices whose origin is in $L_2$ to be different by a constant translation in $[1, 0]^T$. Namely:
\begin{equation}
\tilde{v} - v = a 
\end{equation} where $v$ and $\tilde{v}$ are twins, and $a$ is either $[0, 1]^T$  or $[1, 0]^T$, depending on whether the origin of $v$ belongs to $L_1$ or $L_2$.

Finally we set each vertex $v\in \partial V_{T_{disk}}\setminus{\{v_1', v_2', v_3', v_4'\}}$ to be the weighted average of both its neighbors and the translated neighbors of its twin. \begin{equation}
\sum_{u \in N(v)} w_{uv}(x_v-x_u) + \sum_{u \in N(\tilde{v})} w_{u\tilde{v}}(x_{\tilde{v}} -x_u + a) = 0 
\end{equation} with $a$ as before.

\begin{table*}[t]
	\centering
	\caption{Channel sizes of our U-Net architecture for surface segmentation}
	\begin{tabular}{lllll}
		\toprule
		Spatial Dimensions & Layer     & kernel size & \# input channels & \# output channels \\
		\midrule
		512 x 512          & Conv2d    & 5           & 3                 & 128                \\
		& Conv2d    & 3           & 128               & 128                \\
		& MaxPool2d & 2           &                   &                    \\
		256 x 256          & Conv2d    & 3           & 128               & 128                \\
		& Conv2d    & 3           & 128               & 128                \\
		& MaxPool2d & 2           &                   &                    \\
		128 x 128          & Conv2d    & 3           & 128               & 128                \\
		& MaxPool2d & 2           &                   &                    \\
		64 x 64            & Conv2d    & 3           & 128               & 256                \\
		& MaxPool2d & 2           &                   &                    \\
		32 x 32            & Conv2d    & 3           & 256               & 512                \\
		& MaxPool2d & 2           &                   &                    \\
		16 x 16            & Conv2d    & 3           & 512               & 512                \\
		& Conv2d    & 3           & 512               & 512                \\
		& UpSample  &             &                   &                    \\
		32 x 32            & Conv2d    & 3           & 1024              & 256                \\
		& Conv2d    & 3           & 256               & 256                \\
		& UpSample  &             &                   &                    \\
		64 x 64            & Conv2d    & 3           & 512               & 128                \\
		& UpSample  &             &                   &                    \\
		128 x 128          & Conv2d    & 3           & 256               & 128                \\
		& UpSample  &             &                   &                    \\
		256 x 256          & Conv2d    & 3           & 256               & 128                \\
		& UpSample  &             &                   &                    \\
		& Conv2d    & 3           & 256               & 128                \\
		512 x 512          & Conv2d    & 1           & 128               & 8                  \\
		\bottomrule
	\end{tabular}
	\label{tab:unet_arch}
\end{table*}

\end{document}